\documentclass{article}

\usepackage{microtype}
\usepackage{graphicx}
\usepackage{subfigure}
\usepackage{booktabs} % for professional tables

\usepackage{hyperref}

\usepackage[accepted]{icml2024}

\usepackage{amsmath}
\usepackage{amssymb}
\usepackage{mathtools}
\usepackage{amsthm}
\usepackage{multirow}

\usepackage[capitalize,noabbrev]{cleveref}

\theoremstyle{plain}
\newtheorem{theorem}{Theorem}[section]

\newtheorem{corollary}[theorem]{Corollary}
\theoremstyle{definition}
\newtheorem{definition}[theorem]{Definition}

\theoremstyle{remark}

\usepackage[textsize=tiny]{todonotes}

\icmltitlerunning{EfficientZero V2: Mastering Discrete and Continuous Control with Limited Data}

\begin{document}

\twocolumn[
\icmltitle{EfficientZero V2: Mastering Discrete and Continuous Control with Limited Data}



\icmlsetsymbol{equal}{*}

\begin{icmlauthorlist}
\icmlauthor{Shengjie Wang}{equal,thu,sqz,sai}
\icmlauthor{Shaohuai Liu}{equal,thu,texam}
\icmlauthor{Weirui Ye}{equal,thu,sqz,sai}
\icmlauthor{Jiacheng You}{thu}
\icmlauthor{Yang Gao$^\dag$}{thu,sqz,sai}
\end{icmlauthorlist}

\icmlaffiliation{thu}{Institute for Interdisciplinary Information Sciences, Tsinghua University, Beijing, China}
\icmlaffiliation{sqz}{Shanghai Qi Zhi Institute, Shanghai, China}
\icmlaffiliation{sai}{Shanghai Artificial Intelligence Laboratory, Shanghai, China}
\icmlaffiliation{texam}{Texas A\&M University}

\icmlcorrespondingauthor{Yang Gao}{gaoyangiiis@mail.tsinghua.edu.cn}

\icmlkeywords{Machine Learning, ICML}

\vskip 0.3in
]



\printAffiliationsAndNotice{\icmlEqualContribution} % otherwise use the standard text.

\begin{abstract}
    Sample efficiency remains a crucial challenge in applying Reinforcement Learning (RL) to real-world tasks. While recent algorithms have made significant strides in improving sample efficiency, none have achieved consistently superior performance across diverse domains. In this paper, we introduce EfficientZero V2, a general framework designed for sample-efficient RL algorithms.  
    We have expanded the performance of EfficientZero to multiple domains, encompassing both continuous and discrete actions, as well as visual and low-dimensional inputs. With a series of improvements we propose, EfficientZero V2 outperforms the current state-of-the-art (SOTA) by a significant margin in diverse tasks under the limited data setting. EfficientZero V2 exhibits a notable advancement over the prevailing general algorithm, DreamerV3, achieving superior outcomes in \textbf{50} of \textbf{66} evaluated tasks across diverse benchmarks, such as Atari 100k, Proprio Control, and Vision Control.
    % Notably, it surpasses the original EfficientZero by 17\% in the Atari 100k benchmark, and it exceeds Dreamer V3 by 15\% and 30\% in the Proprio Control 100k and Vision Control 200k benchmarks, respectively.
    % \ywr{why we need such a general algorithm? because...}
    % In this paper, we present EfficientZero V2, a general framework for sample efficient RL algorithms. Our work is built upon EfficientZero, which is well-known for its super-human performance on Atari 26 games with 2 hours human replay time. We extend it to multiple domains, including continuous and discrete actions, and visual and low-dimensional inputs. When evaluated on tasks with limited data availability, EfficientZero V2 significantly surpasses the current state-of-the-art (SoTA) by a substantial margin. Specifically, it exceeds EfficientZero by 40\% in Atari 100k benchmark, and surpasses Dreamer V3 by 40\% and 30\% in Proprio Control 100k and Vision Control 200k benchmarks.
    % \ywr{This paragraph should be polished in words.}
\end{abstract}

\section{Introduction}
\label{intro}
Reinforcement learning (RL) has empowered computers to master diverse tasks, such as Go \citep{silver2018general}, video games \citep{ye2021mastering}, and robotics control \cite{hwangbo2019learning,andrychowicz2020learning,akkaya2019solving}. 
However, these algorithms require extensive interactions with their environments, leading to significantly increased time and computational costs \citep{petrenko2023dexpbt,chen2022system}. 
For example, an RL-based controller requires nearly 100M interactions to reorient complex and diverse object shapes using vision input \cite{chen2023visual}.
Furthermore, building certain simulators for daily housework could be challenging. If gathering data in real-world settings, the process tends to be time-consuming and expensive.
Consequently, it is crucial to explore and develop RL algorithms to achieve high-level performance with limited data.

% However, existing algorithms require massive interactions with simulated environments, which hinders the application of RL in tasks like robotic control and autonomous driving due to the expensive data collections in real-world. 
% For example, OpenAI have harnessed RL-based controllers to successfully rearrange blocks and solve Rubik’s cubes like human beings, but used massive amounts of interactive data in simulation and real world. 

\begin{figure}[t]
\centering
\includegraphics[width=0.47\textwidth]{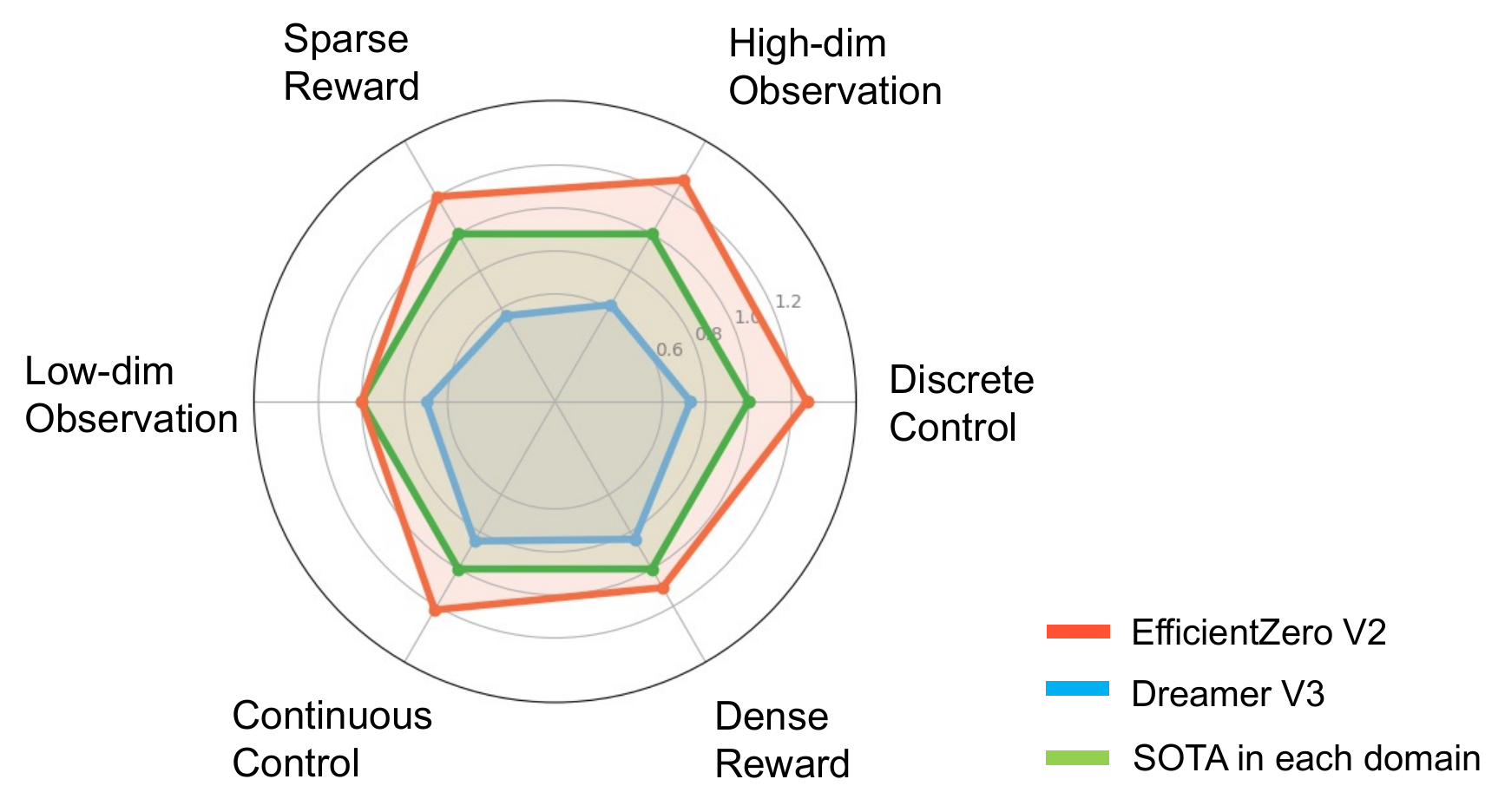}
\caption{Comparison between EfficientZero V2, DreamerV3 and other SOTAs in each domain. We evaluate them under the Atari 100k, DMControl Proprio, and DMControl Vision benchmarks. 
We then set the performance of the previous SOTA as 1, allowing us to derive normalized mean scores for both EfficientZero V2 and Dreamer V3. EfficientZero V2 surpasses or closely matches the previous SOTA in each domain.}
\label{main_sell}
\end{figure}

% For each metric, tasks aligning with the corresponding settings are selected. 

Previous studies have introduced a range of algorithms aimed at enhancing sample efficiency, including TD-MPC series \citep{hansen2022temporal, Anonymous2023TDMPC2}, EfficientZero \citep{ye2021mastering}, and the Dreamer series \citep{hafner2019dream, hafner2021mastering, hafner2023mastering}. Despite these advancements, these algorithms do not consistently attain superior sample efficiency across multiple domains. For example, TD-MPC \citep{hansen2022temporal} leverages Model Predictive Path Integral (MPPI) \citep{rubinstein1997optimization} control for planning. The huge computational burden in planning hinders its application in vision-based RL.
% \ywr{Emphasize TD-MPC fails in image-based RL}. 
EfficientZero \citep{ye2021mastering} employs the Monte-Carlo Tree Search (MCTS) algorithm, which masters discrete control. 
However, EfficientZero is unable to handle high-dimensional action spaces, especially in continuous control scenarios.
DreamerV3 \citep{hafner2023mastering} is a universal algorithm that extends to diverse tasks in a wide range of domains. However, as shown in Fig. \ref{main_sell}, DreamerV3 still has noticeable performance gaps to the state-of-the-art (SOTA) algorithms in each domain.
Thus, there remains an open question on how to achieve both high performance and sample efficiency across various domains at the same time. 
% This approach, however, poses challenges for EfficientZero in managing high-dimensional action spaces in continuous control scenarios.
% The lack of a general algorithm means that we require extra computational resources for choosing and tuning the algorithms when applying RL algorithm to a new application.
% \ywr{why this is a fundamental question?}

In this paper, we propose EfficientZero-v2 (EZ-V2), which can master tasks across various domains with superior sample efficiency.
EZ-V2 successfully extends EfficientZero's strong performance to continuous control, demonstrating strong adaptability for diverse control scenarios. 
The main contributions of this work are as follows. 
\begin{itemize}
    \item We propose a general framework for sample-efficient RL. Specifically, it achieves consistent sample efficiency for discrete and continuous control, as well as visual and low-dimensional inputs.
    % \ywr{be consistent.}
    \item We evaluate our method in multiple benchmarks, outperforming the previous SOTA algorithms under limited data. As shown in Fig.\ref{main_sell}, the performance of EZ-V2 exceeds DreamerV3, a universal algorithm, by a large margin covering multiple domains with a data budget of 50k to 200k interactions. 
    % \ywr{Do not use 's. such as the performance of EZ-V2.}
    \item We design two key algorithmic enhancements: a sampled-based tree search for action planning, ensuring policy improvement in continuous action spaces, and a search-based value estimation strategy to more efficiently utilize previously gathered data and mitigate off-policy issues.
\end{itemize}

% introduces two key algorithmic enhancements: a sampled-based tree search for action planning, ensuring policy improvement in continuous action spaces, and a search-based value estimation strategy to more efficiently utilize previously gathered data and mitigate off-policy issues.
% Our algorithm is built upon Efficient Zero, owing to the super-human performance of Efficient Zero in Atari 100k benchmark,
% Furthermore, the method reduces the number of simulations in tree search, which benefits the planning in large complex actions spaces. 
% Thus, EZ V2 is a universal RL algorithm that achieves state-of-the-art performance using limited data in both discrete and continuous action spaces and high-dimensional and low-dimensional observation spaces.

\section{Related work}

\subsection{Sample Efficient RL}

Sample efficiency in RL algorithms remains an essential direction for research. Inspired by advances in self-supervised learning, many RL algorithms now employ this approach to enhance the learning of representations from image inputs. 
For instance, CURL \citep{laskin2020curl} employs contrastive learning on hidden states to augment the efficacy of fundamental RL algorithms in image-based tasks. Similarly, SPR \citep{schwarzer2020data} innovates with a temporal consistency loss combined with data augmentations, resulting in enhanced performance.

Furthermore, Model-Based Reinforcement Learning (MBRL) has demonstrated high sample efficiency and notable performance in both discrete and continuous control domains.
SimPLE \citep{kaiser2019model}, by modeling the environment, predicts future trajectories, thereby achieving commendable performance in Atari games with limited data.
TD-MPC \citep{hansen2022temporal} utilizes data-driven Model Predictive Control (MPC) \citep{rubinstein1997optimization} with a latent dynamics model and a terminal value function, optimizing trajectories through short-term planning and estimating long-term returns. The subsequent work, TD-MPC2 \citep{Anonymous2023TDMPC2}, excels in multi-task environments. 
The TD-MPC series employs MPC to generate imagined latent states for action planning. In contrast, our method employs a more efficient action planning module called Sampling-based Gumbel Search, leading to lower computational costs.
% In contrast, the TD-MPC series employs MPC to generate imagined latent states for action planning that are 200 times larger than those used in our method, leading to increased computational costs.
% Other works incorporate uncertainty as an intrinsic reward to enhance efficiency \citep{wang2023coplanner} and leverage latent temporal consistency to concurrently learn a state encoder and a latent dynamics model \citep{zhao2023simplified}. 
Dreamer \citep{hafner2019dream}, a reinforcement learning agent, develops behaviors from predictions within a compact latent space of a world model. Its latest iteration, Dreamer V3 \citep{hafner2023mastering}, is a general algorithm that leverages world models and surpasses previous methods across a wide range of domains. It showcases its sample efficiency by learning online and directly in real-world settings \citep{daydreamer}. 
% The differences between DreamerV3 and our method are twofold. First, DreamerV3 employs a reconstruction loss with respect to raw images. However, this reconstruction loss includes unnecessary information from the raw images. These unnecessary information commonly does not contribute to decision-making. Secondly, without the action planning module, DreamerV3 shows inferior performance on the Atari 100k and Vision Control benchmarks.
Although long-horizon planning in Dreamer V3 \citep{hafner2023mastering} enhances the quality of the collected data, an imagination horizon of $H=15$ may be excessively long, potentially leading to an accumulation of model errors.

\subsection{MCTS-based Work} 

% \ywr{Muzero-series work seems a little narrow. MCTS-based RL?}
AlphaGo \citep{silver2016mastering} is the first algorithm to defeat a professional human player in the game of Go, utilizing Monte-Carlo Tree Search (MCTS) \citep{coulom2006efficient} along with deep neural networks. 
AlphaZero \citep{silver2017masteringchess} extends this approach to additional board games such as Chess and Shogi. 
MuZero \citep{schrittwieser2020mastering}, aspiring to master complex games without prior knowledge of their rules, learns to predict game dynamics by training an environment model.
Building upon MuZero, EfficientZero \cite{ye2021mastering} achieves superhuman performance in Atari games with only two hours of real-time gameplay, attributed to the self-supervision of the environment model. 
However, applying MuZero to tasks with large action spaces significantly increases the computational cost of MCTS due to the growing number of simulations. Gumbel MuZero \citep{danihelka2021policy} effectively diminishes the complexity of search within vast action spaces by implementing Gumbel search, although it does not extend to continuous action domains. Sample MuZero \cite{hubert2021learning} proposes a sampling-based MCTS that contemplates subsets of sampled actions, thus adapting the MuZero framework for continuous control. Recent developments have also seen MuZero applied in stochastic environments \citep{antonoglou2021planning} and its value learning augmented by path consistency (PC) optimality regularization \citep{zhao2022efficient}. 
Our method notably enhances Gumbel search for continuous control and requires only half the number of search simulations compared to Sample MuZero \citep{hubert2021learning}.

\begin{figure*}[t]
\centering
\includegraphics[width=0.98\textwidth]{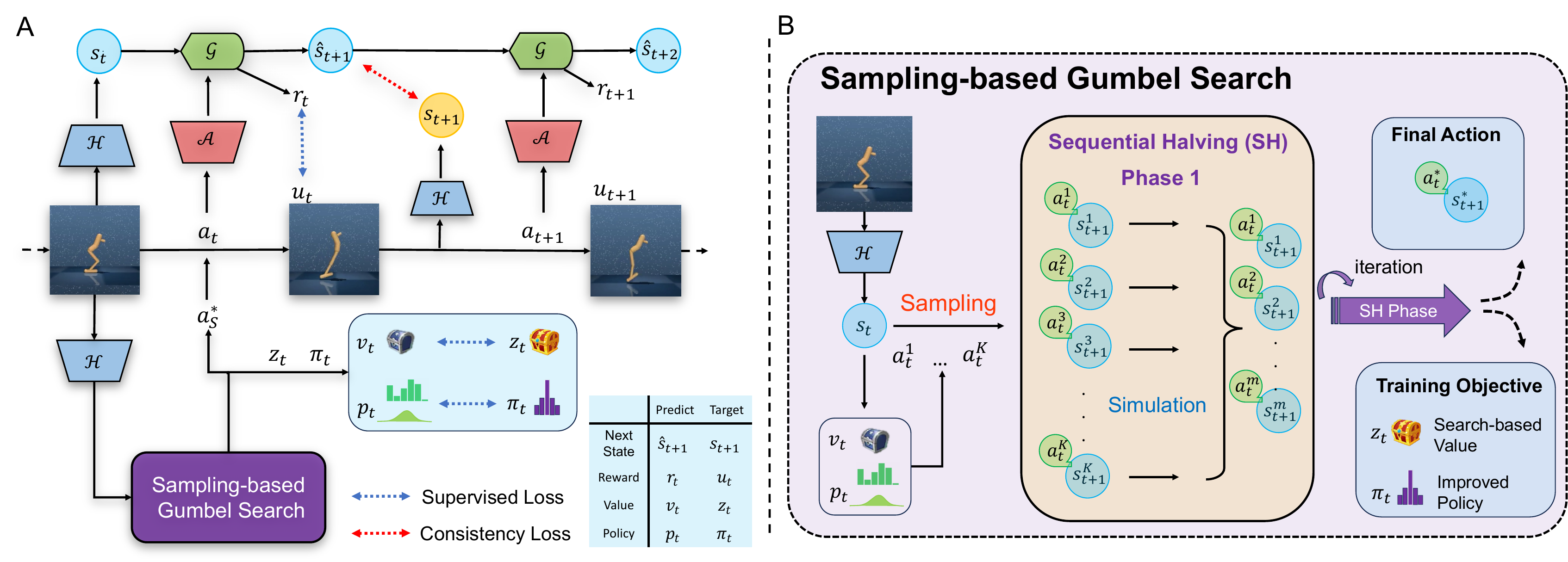}
\caption{Framework of EZ-V2. (\textbf{A}) How EZ-V2 trains its model. The representation \(\mathcal{H}\) takes observations as inputs and outputs the state. The dynamic model \(\mathcal{G}\) predicts the next state and reward based on the current state and action. Sampling-based Gumbel search outputs the target policy \(\pi_t\) and target value \(z_t\). (\textbf{B}): How the sampling-based Gumbel search uses the model to plan. The process contains action sampling and selection. The iterative action selection outputs the recommended action $a^*_S$, search-based value target (target value), and improved policy (target policy).}
\label{framework}
\end{figure*}

\section{Preliminary}

\subsection{Reinforcement Learning}
Reinforcement Learning (RL) can be formulated as a Markov Decision Process (MDP) \citep{bellman1957markovian}. An MDP in this context is formalized as a tuple \( (S, A, T, R, \gamma) \), where \( s \in S \) represents states, \( a \in A \) denotes actions, \( T: S \times A \rightarrow S \) is the transition function, \( R: S \times A \rightarrow \mathbb{R} \) is the reward function associated with a particular task, and \( \gamma \) is the discount factor. 
The goal in RL is to find an optimal policy \( \pi^* \) that maximizes the expected sum of discounted rewards over time, formally expressed as \( \max_\pi \mathbb{E} \left[ \sum_{t=0}^{\infty} \gamma^t R(s_t, a_t) \right] \), where \( a_t \) is an action taken according to the policy \( \pi \) at state \( s_t \). 
In this work, to improve sample efficiency, we learn a model of the environment during training. Meanwhile, planning with the learned model makes action selection more efficient.

\subsection{Gumbel-Top-k Trick}
\label{main_gumbel_topk}
The Gumbel-Top-k trick \citep{kool2019stochastic} can choose the top $n$ actions without replacement in a categorical distribution $\pi$. Specifically, the action sample $A$ can be sampled by the Gumbel-Max trick, which is defined as follows.
\begin{equation}
\begin{aligned}
\left(g \in \mathbb{R}^k\right) & \sim \operatorname{Gumbel}(0) \\
A & =\underset{a}{\arg \max }(g(a)+\operatorname{logits}(a)) 
\end{aligned}
\end{equation}
where $\operatorname{logits}(a)$ is the logit of the action a, and $g$ is a vector of $k$ Gumbel variables. Hereafter, we can sample $n$ actions without replacement. 
\begin{equation}
\begin{aligned}
A_1 & =\underset{a}{\arg \max }(g(a)+\operatorname{logits}(a)) \\
\vdots & \\
A_n & =\underset{a \notin\left\{A_1, \ldots, A_{n-1}\right\}}{\arg \max }(g(a)+\operatorname{logits}(a)) .
\end{aligned}
\end{equation}
Furthermore, we denote the set of $n$ top actions by $\operatorname{argtop}(g + logits, n) = {A_1, A_2,..., A_n}$.

\subsection{EfficientZero}

% Model-based learning becomes a general method to improve the sample efficiency in RL algorithms compared with model-free learning. 
% % It enhances the accuracy of action evaluation by generating potential future trajectories through a model, rather than solely relying on information from the executed trajectory. 
% EZ-V2 also benefits from the model-based learning. To achieve this, 
\subsubsection{Newtork Structure}
EfficientZero algorithm learns a predictive model in a latent space and then performs planning over actions using this model. Specifically, the components of EfficientZero consist of a representation function, a dynamic function, a policy function, and a value function, which are formulated as follows.
\begin{itemize}
    \item Representation Function:  $\mathcal{H}:s_t=\mathcal{H}(o_t)$
    \item Dynamic Function: $\mathcal{G}: \hat{s}_{t+1}, r_t =\mathcal{G}(s_t,a_t)$ 
    \item Policy Function: $\mathcal{P}:p_t=\mathcal{P}(s_t)$
    \item Value Function : $\mathcal{V}: v_t = \mathcal{V}(s_t)$
\end{itemize}
Among these components, \( o_t \) represents the current observation, \( s_t \) is the current latent state, and \( \hat{s}_{t+1} \) is the predicted next state. The representation function learns a compact state representation of the input \( o_t \). The dynamic function predicts the next state \( \hat{s}_{t+1} \) and the reward \( r_t \). 
The policy function outputs the current policy \( p_t \), and the value function provides the value estimation \( v_t \) at the current state. All components are implemented as neural networks. EZ-V2 uses a similar network structure (see Figure 2), with details of the architecture for each component provided in Appendix \ref{arch}. The main difference is that the action embedding block \( \mathcal{A} \) encodes the action \( a_t \) as a latent vector within the dynamic function.

\subsubsection{Training Process}
For the dynamic function, EfficientZero employs a supervised learning method using the true reward \( u_t \). Furthermore, EfficientZero introduces temporal consistency \citep{ye2021mastering}, which strengthens the supervision between the predicted next state \( \hat{s}_{t+1} \) and the true next state \( s_{t+1} \).
Using the learned model, an MCTS method is used for planning actions. The method can generate the target policy \( \pi_t \) and target value \( z_t \) for further supervised learning of the policy and value functions. 
Simultaneously, it outputs the recommended action \( a^*_S \) for interaction with the environment. The iterative process of interaction and training enhances the accuracy of the dynamic function's predictions and gradually improves the policy and value functions.
EZ-V2 inherits the training process of EfficientZero but replaces the MCTS method with a sampling-based Gumbel search, ensuring policy improvement in continuous action spaces. Figure \ref{framework} (A) intuitively illustrates the training process of EZ-V2.

All parameters of the components are trained jointly to match the target policy, value, and reward. 
\begin{equation}
\label{loss}
    \begin{split}
        \mathcal{L}_t &= \lambda_1\mathcal{L}_\mathcal{R}(u_t,r_t)+\lambda_2\mathcal{L}_\mathcal{P}(\pi_t,p_t)\\ &+\lambda_3\mathcal{L}_\mathcal{V}(z_t,v_t)+\lambda_4\mathcal{L}_\mathcal{G}(s_{t+1},\hat{s}_{t+1}) \\
    \end{split}
\end{equation}
where $u_t$ denotes the environmental reward, $\pi_t$ is the target policy from the search and $z_t$ represents the target value from the search. $\mathcal{L}_\mathcal{R}$, $\mathcal{L}_\mathcal{P}$ and $\mathcal{L}_\mathcal{V}$ all represent supervised learning losses.
$\mathcal{L}_\mathcal{G}$ is the temporal consistency loss \cite{schwarzer2020data}, which enhances the similarity between the predicted next state $\hat{s}_{t+1}$ with the next state $s_{t+1}$ and is formulated as: 
\begin{equation}
\label{loss}
    \begin{split}
        \mathcal{L}_{\mathcal{G}}\left(s_{t+1}, \hat{s}_{t+1}\right)&=\mathcal{L}_{cos}\left(s g\left(P_1\left(s_{t+1}\right)\right), P_2\left(P_1\left(\hat{s}_{t+1}\right)\right)\right)
    \end{split}
\end{equation}
where $\mathcal{L}_{cos}$ is the negative cosine similarity loss and $sg$ means the stop gradient operator.
The asymmetric design of employing $P_1$ and $P_2$ follows the setting in SimSiam \citep{chen2021exploring}. More details about the architecture also can be found in EfficientZero \citep{ye2021mastering}.
To enhance the prediction accuracy of the model, we unroll the loss with $l_{\text {unroll }}$ steps to train the model.
\begin{equation}
\label{loss}
    \begin{split}
    \mathcal{L} &=\frac{1}{l_{\text {unroll }}} \sum_{i=0}^{l_{\text {unroll }}-1} \mathcal{L}_{t+i} \\
    \end{split}
\end{equation}
More details of the training pipeline can be found in Appendix \ref{pipeline}.

\section{Method}

\subsection{Overview}
% 这里主要介绍EZ是个啥，然后我们的算法实现了XXX
% 
% EZ-V2, built upon EfficientZero, serves as a general framework for sample-efficient RL across various types of domains. EfficientZero is a model-based algorithm that performs the planning using Monte-Carlo Tree Search in a learned environment model. 
EZ-V2 is built upon EfficientZero, a model-based algorithm that performs planning using MCTS within a learned environment model. EZ-V2 successfully extends EfficientZero's high sample efficiency to various domains. To realize this extension, EZ-V2 addresses two pivotal questions:
\begin{itemize}
\item \textit{How to perform efficient planning using tree search in high-dimensional and continuous action spaces?}
\item \textit{How to further strengthen the ability to utilize stale transitions under limited data?}
\end{itemize}
Specifically, we propose a series of improvements. We construct a sampling-based tree search for policy improvement in continuous control. Furthermore, we propose a search-based value estimation method to alleviate off-policy issues in replaying stale interaction data. The differences compared to EfficientZero are detailed in Appendix \ref{summary_diff}.
% In the following, we introduce the EZ-V2 algorithm in detail.

\subsection{Policy Learning with Tree Search} 
% \wsj{under writing}
% inner policy improvement 
% outer policy improvement
% 改名字 policy improvement with tree search; policy learning
% 去掉对比部分 放到discussion里
% 第一部分讲清楚tree search blabla 然后我们设计了bandit算法 similar to muzero 很好探索xxx； 然后 它没解决的continuous 环境； inspired sample muzero 然后blabla；
% 讲完之后说提出细节的讲这两个部分 

The policy learning in EZ-V2 consists of two stages: (i) \textit{obtaining the target policy from tree search} and (ii) \textit{supervised learning using the target policy}. The tree search method we propose guarantees policy improvement as defined in Definition \ref{def:PI} and enhances the efficiency of exploration in a continuous action space. The training objective aims to refine the policy function by aligning it with the target policy obtained from the tree search.

\begin{definition}[\textbf{Policy Improvement}]
\label{def:PI}
\textit{A planning method over actions satisfies policy improvement if the following inequality holds at any given state \( s \).}
\begin{equation}
\label{PI}
     q(s,a^*_{S}) \geq \mathbb{E}_{a \sim p_t}[q(s,a)]
\end{equation}
\textit{where \( a^*_S \) is the recommended action from the planning method, \( p_t \) is the current policy, and \( q \) is the Q-value function with respect to \( p_t \).}
\end{definition}

% \textit{If the tree search’s action selection produces a policy improvement, then}
% The left-side of the expectation operator means the planning method can be stochastic.

\subsubsection{Target Policy from Tree Search}

% 介绍pipeline 简单介绍过程
% introduce a search  based on gumbel 

% Due to the introduction of model-based learning, we can obtain a predicted model of the environment. By generating the imagined trajectories, we can realize the long-term evaluation of the actions obtained from the policy function. 
% Expanding nodes in tree search methods resembles generating possible trajectories. 
% Thus, we design a tree search method, namely sampling-based Gumbel search. 

In this paper, we choose the tree search method as the improvement operator, which can construct a locally superior policy over actions (policy improvement) based on a learned model. Each node of the search tree is associated with a state \( s \), and an edge is denoted as \( (s,a) \). The tree stores the estimated Q-value of each node and updates it through simulations. Finally, we select an action that strikes a balance between exploitation and exploration, based on the Q-value.

More specifically, the basic tree search method we adopt is the Gumbel search \citep{danihelka2021policy}. This method is recognized for its efficiency in tree searching and its guarantee of policy improvement. At the onset of a search process, the Gumbel search samples \( K \) actions using the Gumbel-Top-k trick (Section \ref{main_gumbel_topk}). It then approaches the root action selection as a bandit problem, aiming to choose the action with the highest Q-value. To evaluate the set of sampled actions, we employ a bandit algorithm known as Sequential Halving \citep{karnin2013almost}. However, the Gumbel search primarily investigates planning in discrete action spaces.

% represents the maximum point of the Q-value at the root node. 

% More specifically, the basic tree search method we choose is the Gumbel search \citep{danihelka2021policy}. It is an efficient tree search method and guarantees policy improvement. At the beginning of a search process, Gumbel search samples $K$ actions using Gumbel-Top-k trick (Appendix \ref{gumbel_topk}). Then it formulates the root action selection as a bandit problem. The objective of the bandit is to choose the action with the highest Q-value.
% A bandit algorithm, Sequential Halving, was employed to progressively evaluate the set of sampled actions. 
% However, their method mainly investigates the planning in discrete action space.
% A $K$-armed bandit is a vector of Q-values of $K$ actions.
% The action selection process is equivalent to find the argument of maximum of the Q-value manifold, thus maintaining the policy improvement. 
% To support high-dimensional continuous control, we design a sampling-based Gumbel search. Fig. \ref{framework} (B) illustrates the process of the sampling-based Gumbel search.

% 这里需要讲的更有趣一些

% In the following, we focus on the procedure in continuous control. we propose an action sampling method that achieves excellent exploration ability, as well as makes the search method satisfy policy improvement in Definition \ref{def:PI}.
To support high-dimensional continuous control, we design a sampling-based Gumbel search, as depicted in Fig. \ref{framework} (B). Given the challenges posed by high-dimensional and large continuous action spaces, striking a balance between exploration and exploitation is crucial for the performance of tree search, especially when using a limited number of sampled actions. To address this challenge, we propose an action sampling method that not only achieves excellent exploration capabilities but also ensures that the search method satisfies policy improvement as defined in Definition \ref{def:PI}.

% Given any state $s$, we sample $K$ actions, which form an action set $A_S$. 
% The action set $A_S$ consists of two action sets, $A_{S1}$ and $A_{S2}$.
% $A_{S1}$ is sampled from the current policy $p_t$, and the action set $A_{S2}$ comes from any other distribution. 
Our method is implemented as follows. Given any state \( s \), we sample \( K \) actions. A portion of these actions originates from the current policy \( p_t \), while another portion is sampled from a prior distribution \( p^\prime_t \). We denote the complete action set as \( A_S = [A_{S1}, A_{S2}] \), where \( A_{S1} \) and \( A_{S2} \) represent the two portions, respectively. This design enhances exploration because \( A_{S2} \) can introduce actions that have a low prior under the current policy \( p_t \). 
The bandit process then selects the action \( a^*_S \) from the action set \( A_S \) with the highest \( q(s,a) \), expressed as \( a^*_S = \arg\max_{a \in A_S}(q(s,a)) \). This process guarantees policy improvement as outlined in Definition \ref{def:PI}, because:
\begin{equation}
\label{q_pi}
\begin{split}
    q(s,a^*_S ) & \geq \max \left( \frac{\sum_{a \in A_{S1}} q(s,a)}{|A_{S1}|}, \frac{\sum_{a \in A_{S2}} q(s,a)}{|A_{S2}|}  \right) \\
    & \geq \frac{\sum_{a \in A_{S1}} q(s,a)}{|A_{S1}|} \\
    & = \mathbb{E}_{a\sim p_t }[q(s, a)] , \ \  \text{as} \ \  |A_{S1}| \rightarrow \infty
\end{split}
\end{equation}
where \( |A_{S1}| \) represents the number of actions in \( A_{S1} \). The first line holds because \( a^*_S \) is the action with the highest \( q(s,a) \) in \( A_S \) and \( [A_{S1}, A_{S2}] = A_S \). We apply the law of large numbers to transition from line 2 to 3.
In practical implementation, we model the current policy \( p_t \) as a Gaussian distribution, and the sampling distribution \( p^\prime_t \) is a flattened version of the current policy. Our experiments demonstrate that this design facilitates exploration in continuous control.

% This holds for any $A_S$ sampled from the current policy $p_t$. 
% Meanwhile, only if $A_S$ contains the samplings from the current policy $p_t$, Equation \eqref{PI} holds for expectation. 
% Intuitively, when $A_S$ contains all samples in $p_t$, the highest Q-value in $A_S$ is better than the expectation of Q-value under $p_t$.
% That brings an advantage for exploration. 
% We can add some actions with low prior under the current policy, and the policy improvement still holds. 
% In practical implementation, we model the current policy $p_t$ as a Gaussian distribution, and the prior distribution $P$ is a mixed Gaussian distribution. 
% In practice, we find that sampling $A_S'$ from a wilder distribution $p_t'$ is beneficial for exploration.
% Sampling from a wilder distribution means a portion of the actions comes from the current policy $p_t$, while another portion comes from the current policy with 3 times standard derivation. 
% Furthermore, sampling from a wilder distribution still guarantees policy improvement if and only if $p_t'$ envelopes $p_t$ and $A_S\subset A_S'$, because
% \begin{equation}
% \label{q_pi}
%     q(s,\underset{a \in A_S'\sim p_t'}{\arg \max}(\sigma(q(s,a))) )\geq \mathbb{E}_{a\in A_S\sim p_t}q(s, a)
% \end{equation}
% reduce variance for 

For action sampling at non-root nodes, we modify the sampling method to reduce the variance in the estimation of Q-values. Specifically, actions at non-root nodes are sampled solely from the current policy \( p_t \). Additionally, the number of actions sampled at non-root nodes is fewer than those at the root node. This reduction in the number of sampled actions at non-root nodes facilitates an increase in search depth, thereby avoiding redundant simulations on similar sampled actions.

Upon completing the sampling-based Gumbel search, we obtain the target policy. We construct two types of target policies. The first is the recommended action \( a^*_S \) obtained from the tree search. In addition to \( a^*_S \), the search also yields \( q(s,a) \) for the visited actions. To smooth the target policy, we build the second type as a probability distribution based on the Q-values of root actions (for more details, see Appendix \ref{cal_ip}).

\subsubsection{Learning using Target Policy}
In this part of the process, we distill the target policy \( \pi_t \) into the learnable policy function \( p_t \). We aim to minimize the cross-entropy between \( p_t \) and the target policy \( \pi_t \):
\begin{equation}
\label{ce_loss}
    \mathcal{L}_{\mathcal{P}}(p_t, \pi_t) = \mathbb{E}_{a \sim \pi_t}\left[-\log p_t (a) \right]
\end{equation}
Additionally, in high-dimensional action spaces, we utilize the other target \( a^*_S \):
\begin{equation}
    \label{simple_loss}
    \mathcal{L}_{\mathcal{P}}(p_t, a^*_{S})=-\log p_t\left(a^*_{S}\right)
\end{equation}
Compared with Equation \eqref{ce_loss}, Equation \eqref{simple_loss} facilitates early exploitation in tasks with a large action dimension, such as the `Quadruped walk' in DM Control \citep{tassa2018deepmind}. We provide an intuitive example to illustrate its advantage in Appendix \ref{simplepi}.

% In this part, we distil the target policy $\pi_t$ into the learnable policy function $p_t$. We minimize the cross-entropy between $p_t$ and the target policy $\pi_t$:
% \begin{equation}
% \label{ce_loss}
%     \mathcal{L}_{\mathcal{P}}(\pi_t,p_t) = \mathbb{E}_{a \sim \pi_t}\left[-\log p_t (a) \right]
% \end{equation}
% Additionally, we use the other target $a^*_S$ in high-dimensional action space:
% \begin{equation}
%     \label{simple_loss}
%     \mathcal{L}_{\mathcal{P}}(p_t)=-\log p_t\left(a^*_{S}\right)
% \end{equation}
% Compared with Equation \eqref{ce_loss}, Equation \eqref{simple_loss} benefits the early exploitation in the task with a large action dim, such as Quadruped walk in DM control. We provide an intuitive example to illustrate its advantage in Appendix \ref{simplepi}.

\subsection{Search-based Value Estimation}

Improving the ability to utilize off-policy data is crucial for sample-efficient RL. Sample-efficient RL algorithms often undergo drastic policy shifts within limited interactions, leading to estimation errors for early-stage transitions in conventional methods, such as \( N \)-step bootstrapping and TD-\( \lambda \). EfficientZero proposed an adaptive step bootstrapping method to alleviate the off-policy issue. 
However, this method utilizes the multi-step discount sum of rewards from an old policy, which can lead to inferior performance. Therefore, it is essential to enhance value estimation to better utilize stale transitions.

We propose leveraging the current policy and model to conduct value estimation, which we term \textbf{Search-Based Value Estimation (SVE)}. The expanding search tree generates imagined trajectories that provide bootstrapped samples for root value estimations. We now use the mean of these empirical estimations as target values. Notably, this value estimation method can be implemented within the same process as the policy reanalysis proposed by MuZero \citep{schrittwieser2021online}, thereby not introducing additional computational overhead.
% , which has been applied in EfficientZero
% We name this method as \textbf{Search-Based Value Estimation (SVE)}. 
% SVE improves value estimations by utilizing current policy, as well as the learned dynamic model, to re-estimate empirical values through a meticulous search process, as shown in the bottom of Fig. \ref{framework} (B). 
The mathematical definition of SVE is as follows.
% The specific calculation can be found in Appendix \ref{cal_sve}.
\begin{definition}[\textbf{Search-Based Value Estimation}]
\label{def:sve}
\textit{Using imagined states and rewards $\hat{s}_{t+1}, \hat{r}_t=\mathcal{G}(\hat{s}_t,\hat{a}_t)$ obtained from our learnable dynamic function, 
% the search-based value estimation of a given state $s_0$ is defined as
the value estimation of a given state $s_0$ can be derived from the empirical mean of $N$ bootstrapped estimations, which is formulated as
}
\begin{equation}
    \hat{V}_\text{S}(s_0)=\frac{\sum_{n=0}^{N}\hat{V}_n(s_0)}{N}
\end{equation}
\textit{where $N$ denotes the number of simulations, $\hat{V}_n(s_0)$ is the bootstrapped estimation of the $n$-th node expansion, which is formulated as }
\begin{equation}            \hat{V}_n(s_0)=\sum_{t=0}^{H(n)}\gamma^t\hat{r}_t+\gamma^{H(n)}\hat{V}(\hat{s}_{H(n)})
\end{equation}
\textit{where $H(n)$ denotes the search depth of the $n$-th iteration.}
\end{definition}

Through the imagined search process with the newest policy and model, SVE provides a more accurate value estimation for off-policy data. Furthermore, investigating the nature of estimation errors is critical. We derive an upper bound for the value estimation error, taking into account model errors, as illustrated in Theorem \ref{theorem:sve_error}.

% Inspired by the proof in \citep{feinberg2018model}, 

% \begin{theorem}[\textbf{Search-Based Value Estimation Error}]
% \label{theorem:sve_error}
\begin{corollary}[\textbf{Search-Based Value Estimation Error}]
\label{theorem:sve_error}
    Define $s_t,a_t,r_t$ to be the states, actions, and rewards resulting from current policy $\pi$ using true dynamics $\mathcal{G}^*$ and reward function $\mathcal{R}^*$, starting from $s_0\sim\nu$ and similarly define $\hat{s}_t, \hat{a}_t, \hat{r_t}$ using learned function $\mathcal{G}$. Let reward function $\mathcal{R}$ to be $L_r-Lipschitz$ and value function $\mathcal{V}$ as $L_V-Lipschitz$. Assume $\epsilon_s, \epsilon_r, \epsilon_v$ as upper bounds of state transition, reward, and value estimations respectively. We define the error bounds of each estimation as
    % \begin{equation}
    %     \max_{n\in[N],t\in[H(n)]}\mathbb{E}\left[\Vert\hat{s}_t-s_t\Vert^2\right]\leq\epsilon^2
    % \end{equation}
    \begin{gather}
        \max_{n\in[N],t\in[H(n)]}\mathbb{E}\left[\Vert\hat{s}_t-s_t\Vert^2\right]\leq\epsilon_s^2 \\
        \max_{n\in[N],t\in[H(n)]}\mathbb{E}\left[\Vert\mathcal{R}(s_t)-\mathcal{R}^*(s_t)\Vert^2\right]\leq\epsilon_r^2 \\
        \max_{n\in[N],t\in[H(n)]}\mathbb{E}\left[\Vert\mathcal{V}(s_t)-\mathcal{V}^*(s_t)\Vert^2\right]\leq\epsilon_v^2
    \end{gather}
    within a tree-search process. Then we have errors
    \begin{equation}
    \begin{aligned}
        &\text{MSE}_\nu(\hat{V}_\text{S})\\
        &\leq\frac{4}{N^2}\sum_{n=0}^N\left(\sum_{t=0}^{H(n)}\gamma^{2t}(L_r^2\epsilon_s^2+\epsilon_r^2)+\gamma^{2H(n)}(L_V^2\epsilon_s^2+\epsilon_v^2)\right)
    \end{aligned}
    \end{equation}
    % Notably, the coefficient of $\epsilon^2$ 
    % \begin{equation*}
    %     \frac{2}{N^2}\sum_{n=0}^N\left(\sum_{t=0}^{H(n)}\gamma^{2t}L_r^2+\gamma^{2H(n)}L_V^2\right)
    % \end{equation*}
    % is convergent with $H(n)$. Additionally, 
    where $N$ is the simulation number of the search process. $H(n)$ denotes the depth of the $n$-th search iteration.
    % The upper bound will converge to 0 when the dynamic function is approximately optimal $\epsilon\to 0$.
\end{corollary}

The detailed proof can be found in Appendix \ref{app:theorem:sve_error}. SVE possesses several advantageous properties, such as a convergent series coefficient and bounded model errors. Intuitively, the estimation error bound will converge to 0 when the dynamic function approaches optimality, denoted as \(\epsilon \to 0\).

% Although the convergence speed of the dynamic funciton varies for tasks with different difficulty levels, we demonstrate that the SVE method consistently outperform the previous adaptive bootstrap method proposed in EfficientZero \citep{ye2021mastering}.

% However, the model inaccuracy in the early training stage will introduce significant errors into SVE, which leads to unsatisfying performance. In addition, the SARSA targets are not inferior to SVE when replaying recently collected data since this situation could be considered as on-policy.
% The convergence speed of the dynamic function varies for tasks with different difficulty levels, which could introduce model errors into value estimations in the early training stage. The direct usage of search-based estimations could harm the performance in the whole training. 
Theorem \ref{theorem:sve_error} shows that model inaccuracies can amplify SVE's estimation error, especially in the early training stages or when sampling fresh transitions. To address this, we introduce a mixed value target, combining multi-step TD-targets for early training and fresh experience sampling. The mixed target is defined as:
\begin{equation}
    V_{mix}=\begin{cases}
        \sum_{i=0}^{l-1}\gamma^i u_{t+i}+\gamma^l v_{t+l}, & \text{if } i_t<T_1 \  \\ &\text{or} \ i_s>|D|-T_2\\
        \hat{V}_\text{S}, & \text{otherwise}
    \end{cases}
\end{equation}
Here, \( l \) is the horizon for the multi-step TD method. The variable \( i_t \) denotes the current training step, while \( T_1 \) refers to the initial steps. The term \( i_s \) indicates the sampling index from the buffer. The buffer size is represented by \( |D| \), and \( T_2 \) is the designated threshold for assessing the staleness of data. More details can be found in Appendix \ref{mixedV}.

\begin{table*}[ht]
    \caption{Scores achieved on the Atari 100k benchmark indicate that EZ-V2 achieves super-human performance within just 2 hours of real-time gameplay. Our method surpasses the previous state-of-the-art, EfficientZero. The results for Random, Human, SimPLe, CURL, DrQ, SPR, MuZero, and EfficientZero are sourced from \citep{ye2021mastering}.}
    \label{tab:atari_results_full}
\begin{center}
\begin{small}
% \begin{sc}
\centering
\scalebox{0.8}{
\centering
\begin{tabular}{lccccccccccr}
\toprule
Game &                  Random &    Human &   SimPLe &     CURL &      DrQ &     SPR & MuZero & EfficientZero & DreamerV3 & BBF & \textbf{EZ-V2 (Ours)}\\
\midrule
Alien               &    227.8 &   7127.7 &    616.9 &   558.2 &    771.2 &   801.5 & 530.0 & 808.5 & 959 & \underline{1173.2} & \textbf{1557.7} \\
Amidar              &      5.8 &   1719.5 &     88.0 &   142.1 &    102.8 &   176.3 & 38.8 & 148.6 & 139 & \textbf{244.6} & \underline{184.9} \\
Assault             &    222.4 &   742.0 &    527.2 &   600.6 &    452.4 &   571.0 & 500.1 & 1263.1 & 706 & \textbf{2098.5} & \underline{1757.5} \\
Asterix             &    210.0 &   8503.3 &   1128.3 &   734.5 &    603.5 &   977.8 & 1734.0 & \underline{25557.8} & 932 & 3946.1 & \textbf{61810.0} \\
Bank Heist          &     14.2 &    753.1 &     34.2 &   131.6 &    168.9 &   380.9 & 192.5 & 351.0 & 649 & \underline{732.9} & \textbf{1316.7} \\
BattleZone          &   2360.0 &  37187.5 &   5184.4 &  14870.0 &  12954.0 & 16651.0 & 7687.5 & 13871.2 & 12250 & \textbf{24459.8} & \underline{14433.3} \\
Boxing              &      0.1 &     12.1 &      9.1 &     1.2 &      6.0 &    35.8 & 15.1 & 52.7 & 52.7 & \textbf{85.8} & \underline{75.0} \\
Breakout            &      1.7 &     30.5 &     16.4 &     4.9 &     16.1 &    17.1 & 48.0 & \textbf{414.1} & 31 & 370.6 & \underline{400.1} \\
ChopperCmd           &    811.0 &   7387.8 &   1246.9 &  1058.5 &    780.3 &   974.8 & \underline{1350.0} & 1117.3 & 420 & \textbf{7549.3} & 1196.6 \\
Crazy Climber       &  10780.5 &  35829.4 &  62583.6 & 12146.5 &  20516.5 & 42923.6 & 56937.0 & 83940.2 & \underline{97190} & 58431.8 & \textbf{112363.3} \\
Demon Attack        &    152.1 &   1971.0 &    208.1 &   817.6 &    1113.4&   545.2 & 3527.0 & 13003.9 & 303 & \underline{13341.4} & \textbf{22773.5} \\
Freeway             &      0.0 &     29.6 &     20.3 &    \textbf{26.7} &      9.8 &    24.4 & 21.8 & 21.8 & 0 & \underline{25.5} & 0.0 \\
Frostbite           &     65.2 &   4334.7 &    254.7 &  1181.3 &    331.1 &  \underline{1821.5} & 255.0 & 296.3 & 909 & \textbf{2384.8} & 1136.3 \\
Gopher              &    257.6 &   2412.5 &    771.0 &   669.3 &    636.3 &   715.2 & 1256.0 & 3260.3 & \underline{3730} & 1331.2 & \textbf{3868.7} \\
Hero                &   1027.0 &  30826.4 &   2656.6 &  6279.3 &   3736.3 &  7019.2 & 3095.0 & 9315.9 & \textbf{11161} & 7818.6 & \underline{9705.0} \\
Jamesbond           &     29.0 &    302.8 &    125.3 &   \underline{471.0} &    236.0 &   365.4 & 87.5 & 517.0 & 445 & \textbf{1129.6} & 468.3 \\
Kangaroo            &     52.0 &   3035.0 &    323.1 &   872.5 &    940.6 &  3276.4 & 62.5 & 724.1 & \underline{4098} & \textbf{6614.7} & 1886.7 \\
Krull               &   1598.0 &   2665.5 &   4539.9 &  4229.6 &   4018.1 &  3688.9 & 4890.8 & 5663.3 & 7782 & \underline{8223.4} & \textbf{9080.0}\\
Kung Fu Master      &    258.5 &  22736.3 &  17257.2 &  14307.8 &   9111.0 & 13192.7 & 18813.0 & \textbf{30944.8} & 21420 & 18991.7 & \underline{28883.3} \\
Ms Pacman           &    307.3 &   6951.6 &   1480.0 &   1465.5 &    960.5 &  1313.2 & 1265.6 & 1281.2 & 1327 & \underline{2008.3} & \textbf{2251.0} \\
Pong                &    -20.7 &     14.6 &     12.8 &     -16.5 &     -8.5 &    -5.9 & -6.7 & \underline{20.1} & 18 & 16.7 & \textbf{20.8} \\
Private Eye         &     24.9 &  69571.3 &     58.3 &   \underline{218.4} &    -13.6 &   124.0 & 56.3 & 96.7 & \textbf{882} & 40.5 & 99.8 \\
Qbert               &    163.9 &  13455.0 &   1288.8 &   1042.4 &    854.4 &   669.1 & 3952.0 & \underline{13781.9} & 3405 & 4447.1 & \textbf{16058.3} \\
Road Runner         &     11.5 &   7845.0 &   5640.6 &  5661.0 &   8895.1 & 14220.5 & 2500.0 & 17751.3 & 15565 & \textbf{33426.8} & \underline{27516.7} \\
Seaquest            &     68.4 &  42054.7 &    683.3 &   384.5 &    301.2 &   583.1 & 208.0 & 1100.2 & 618 & \underline{1232.5} & \textbf{1974.0} \\
Up N Down           &    533.4 &  11693.2 &   3350.3 &  2955.2 &   3180.8 & \textbf{28138.5} & 2896.9 & \underline{17264.2} & - & 12101.7 & 15224.3 \\
\midrule
Normed Mean         &    0.000 &    1.000 &    0.443 &    0.381 &    0.357 &   0.704 & 0.562 & 1.945 & 1.120 & \underline{2.247} & \textbf{2.428} \\
Normed Median       &    0.000 &    1.000 &    0.144 &    0.175 &    0.268 &   0.415 & 0.227 & \underline{1.090} & 0.490 & 0.917 & \textbf{1.286} \\
%Mean DQN@50M-Norm'd     &    0.000 &   23.382 &    0.232 &    0.239 &     0.197 &    0.325 &    0.171 &    0.336 &        \textbf{0.510} \\
%Median DQN@50M-Norm'd  &    0.000 &    0.994 &    0.118 &    0.142 &     0.103 &    0.142 &    0.131 &    0.225 &        \textbf{0.361} \\ \midrule
%\# Superhuman        &        0 &       N/A &        2 &        2 &         1 &        2 &        2 &        5 &            \textbf{7} \\
\bottomrule
%\vskip -0.5cm
\end{tabular}
}
\end{small}
\end{center}
\end{table*}

\begin{table*}[ht]
    \caption{Scores achieved on the Proprio Control 50/100k and Vision Control 100/200k benchmarks (with 3 seeds run for each) demonstrate that EZ-V2 consistently maintains sample efficiency, whether with proprioceptive or visual inputs. The tasks are categorized into easy and hard groups as proposed by \citep{hubert2021learning}. The results of DreamerV3 are sourced from the official data \citep{hafner2023mastering}.}
    \label{tab:dmc_results_full}
\begin{center}
\begin{small}
% \begin{sc}
\centering
\scalebox{0.85}{
\centering
\begin{tabular}{l|cccc|cccc}
\toprule
\multicolumn{1}{l|}{\large Benchmark} & \multicolumn{4}{c|}{\large Proprio Control 50k} & \multicolumn{4}{c}{\large Vision Control 100k} \\
\midrule
Task &   SAC &  TD-MPC2 &   DreamerV3 & \textbf{EZ-V2 (Ours)} &   CURL &  DrQ-v2 &    DreamerV3 & \textbf{EZ-V2 (Ours)}  \\
\midrule
Cartpole Balance             &      \textbf{997.6} &   \underline{962.8}&     839.6 &    947.3 &    \underline{963.3} &    \textbf{965.5} &   956.4 & 911.7  \\
Cartpole Balance Sparse             &    \underline{993.1} &    942.8 &    559.0 &   \textbf{999.2} &    \underline{999.4} &    \textbf{1000.0} &   813.0 & 951.5 \\
Cartpole Swingup            &    \textbf{861.6} &   \underline{826.7} &   527.7 &   805.4 &    \textbf{765.4} &    \underline{756.0} &   374.8 & 747.8  \\
Cup Catch              &      949.9 &    \textbf{976.0} &      729.6&     \underline{969.8} &      932.3&      468.0 &    \underline{947.7} & \textbf{954.7}  \\
Finger Spin            &      \underline{900.0} &     \textbf{965.8} &    765.8 &     837.1 &      \underline{850.2} &     459.4 &    633.2 & \textbf{927.6}  \\
Pendulum Swingup           &     158.9 &   520.1&    \textbf{830.4} &   \underline{825.4} &   144.1 &    233.3 &  \underline{619.3} &  \textbf{726.7}  \\
Reacher Easy           &     744.0 &    \underline{903.5} &    693.4 &   \textbf{940.3} &    467.9 &    \underline{722.1} &   441.4 & \textbf{946.3} \\
Reacher Hard           &     646.5 &   580.4 &    \underline{768.0} &   \textbf{795.4}  &    112.7 &    \underline{202.9} &  120.4 & \textbf{961.5} \\
Walker Stand      &   870.0 &  \textbf{973.9} &  767.3 &  \underline{953.6} &  733.8 &   426.1 & \underline{939.5} & \textbf{944.9}  \\
Walker Walk           &    813.2 &   \textbf{965.5} &   475.2 &   \underline{944.0} &   538.5 &   681.5 &  \underline{771.2} & \textbf{888.8}  \\
\midrule
\multicolumn{1}{l|}{} & \multicolumn{4}{c|}{\large Proprio Control 100k} & \multicolumn{4}{c}{\large Vision Control 200k} \\
\midrule
Acrobot Swingup               &    44.1 &   \textbf{303.2}&    62.8 &   \underline{297.7} &    6.8 &    15.1 &   \underline{67.4} & \textbf{231.8} \\
Cartpole Swingup Sparse         &     256.6 &    \underline{421.4} &     172.7 &   \textbf{795.4} &    8.8 &    81.2 &   \underline{392.4} & \textbf{763.6} \\
Cheetah Run          &   \textbf{680.9} &  614.4 &   400.8 &  \underline{677.8} &  405.1 &  418.4 & \underline{587.3} & \textbf{631.6}  \\
Finger Turn Easy           &    \underline{630.8} &   \textbf{793.3} &   560.5 &  310.7 & \underline{371.5} &    286.8 &   366.6 & \textbf{799.2}  \\
Finger Turn Hard      &  414.0 &  \textbf{604.8} &  \underline{474.2} & 374.1 &  236.3 &  \underline{268.4} & 258.5 & \textbf{794.6} \\
Hopper Hop        &    0.1 &   \underline{84.5} &    9.7 &   \textbf{186.5} &    \underline{84.5}&   26.3&   76.3 & \textbf{206.4} \\
Hopper Stand             &      3.8 &     \textbf{807.9} &     296.1 &    \underline{795.4} &     627.7 &      290.2 &    \underline{652.5} & \textbf{805.7} \\
Quadruped Run              &    139.7 &   \textbf{742.1} &   289.0  &   \underline{510.6} &    170.9 &    \underline{339.4} &   168.0 & \textbf{384.8}  \\
Quadruped Walk               &   237.5 & \underline{853.7} &   256.2 &  \textbf{925.8} &   131.8 &   \underline{311.6} &  122.6 & \textbf{433.3} \\
Walker Run              &   635.4 &   \textbf{780.5} &   478.9 &  \underline{657.2} &   274.7 &   359.9 &  \textbf{618.2} & \underline{475.3}\\
\midrule
\multicolumn{1}{l|}{Mean}         &    552.0 &    \textbf{740.9} &    517.1 &    \underline{723.2} &   437.3 &    410.3 &   \underline{498.5} & \textbf{726.1}  \\
\multicolumn{1}{l|}{Median}        &    633.3 &    \textbf{806.4} &    543.4 &    \underline{800.4} &   324.9 &    330.6 &  \underline{484.5} & \textbf{788.1}    \\

\bottomrule
%\vskip -0.5cm
\end{tabular}
}
\end{small}
\end{center}
\end{table*}

\section{Experiment}
In this section, we aim to evaluate the general sample efficiency of EZ-V2 on a total of 66 diverse tasks.
The tasks include scenarios with low and high-dimensional observation, discrete and continuous action spaces, and dense and sparse rewards. We then present an ablation study on the sampling-based Gumbel search and mixed value target we propose.\footnote{The code is released at \href{https://github.com/Shengjiewang-Jason/EfficientZeroV2}{https://github.com/Shengjiewang-Jason/EfficientZeroV2}.} 

\subsection{Experimental Setup}

To assess sample efficiency, we measure algorithm performance with limited environment steps. In discrete control, we use the \textbf{Atari 100k} benchmark \citep{1606.01540}, encompassing 26 Atari games and limiting training to 400k environment steps, equivalent to 100k steps with action repeats of 4. 
% The observation space is RGB images, and the action space is discrete.

For continuous control evaluation, we utilize the DeepMind Control Suite (DMControl; \citep{tassa2018deepmind}), comprising various tasks in classical control, locomotion, and manipulations. Referring to the categorizations in Sample MuZero \citep{hubert2021learning}, tasks are divided into \textbf{easy} and \textbf{hard} categories. The easy tasks use half the interaction data of the hard tasks. We establish the following benchmarks:
\begin{itemize}
    \item \textbf{Proprio Control 50k} for easy tasks with low-dimensional state inputs.
    \item \textbf{Proprio Control 100k} for hard tasks with low-dimensional state inputs.
    \item \textbf{Vision Control 100k} for easy tasks with image observations.
    \item \textbf{Vision Control 200k} for hard tasks with image observations.
\end{itemize}
Each benchmark includes 10 tasks. Action repeats are set to 2 and the maximum episode length is 1000 for all 4 benchmarks, in line with previous studies \citep{hafner2023mastering, Anonymous2023TDMPC2}.

% \textbf{Proprio Control 50k} includes relatively easy tasks with low-dimensional state (proprioceptive) inputs.

% \textbf{Proprio Control 100k} includes relatively hard tasks with low-dimensional state (proprioceptive) inputs.

% \textbf{Vision Control 100k} includes relatively easy tasks with visual observations.

% \textbf{Vision Control 200k} includes relatively hard tasks with visual observations.
% 
% 'Proprio Control 50/100k' consists of 20 tasks with a maximum budget of 200k frames due to action repeats of 2. 
% Based on the difficulty of tasks, tasks are grouped into categories, easy and hard, respectively. For the easy tasks, we evaluate them with only 100K frames, while the hard tasks are tested with a budget of 200k frames. 
% 'Vision Control 100/200k' includes the same 20 tasks with a maximum budget of 400k frames. Similar to 'Proprio Control 100k', the tasks in 'Vision Control 100/200k' are also categorized into easy and hard groups. The easy tasks are evaluated using only 200k frames with the environment. 

We choose strong baselines for each domain, which include \textbf{SAC} \citep{haarnoja2018soft}, \textbf{DrQ-v2} \citep{yarats2021mastering}, \textbf{TD-MPC2} \citep{Anonymous2023TDMPC2}, \textbf{DreamerV3} \citep{hafner2023mastering}, \textbf{EfficientZero} \citep{ye2021mastering}, and \textbf{BBF} \citep{schwarzer2023bigger}. For more details on the implementation, please refer to Appendix \ref{hyper-param}.

% The hyper-parameters of comparison algorithms have been finetuned in each benchmark. 

\subsection{Comparison with Baselines}

\textbf{Atari 100k}: The performance of EZ-V2 on the Atari 100k benchmark is elaborated in Table \ref{tab:atari_results_full}. When scores are normalized against those of human players, EZ-V2 attains a mean score of 2.428 and a median score of 1.286, surpassing the previous state-of-the-art, BBF \citep{schwarzer2023bigger} and EfficientZero \citep{ye2021mastering}. 
In contrast to BBF, our method employs fewer network parameters and a lower replay ratio.
Such enhancements in performance and computational efficiency are attributed to the learning of the environment model and the implementation of Gumbel search in action planning.
Moreover, EZ-V2 necessitates fewer search simulations compared to EfficientZero and still manages to achieve superior performance.
The utilization of a mixed value target decisively mitigates off-policy issues associated with using outdated data, marking a substantial advancement beyond the adaptive step bootstrapping target used by EfficientZero.

% \textbf{Atari 100k}: Table \ref{tab:atari_results_full} shows the results of EZ-V2 at the Atari 100k benchmark. Normalizing scores with human players, EZ-V2 achieves a mean score of 2.333 and a median score of 1.232, which outperforms the previous SOTA, EfficientZero, by 20\% and 13\% respectively. We provide the training curves in Appendix \ref{train_curves}.
% In addition, EZ-V2 achieves 3 times fewer simulations in search compared to EfficienctZero.
% The improvement of performance and computation efficiency results from the Gumbel search and the mixed value target. Compared with MCTS in EfficientZero, Gumbel search ensures a policy improvement. Besides, the mixed value target benefits the usage of stale data within limited data.

% To be specific, the performance of EZ-V2 is better than that of EfficientZero in some hard tasks, including Alien, Asterix, and Ms Pacman. 

\textbf{Proprio Control}: 
The results in Table \ref{tab:dmc_results_full} showcase that our method achieves a mean score of 723.2 across 20 tasks with limited data. While the performance of the current state-of-the-art, TD-MPC2, is comparable to that of EZ-V2, our method achieves faster inference times. TD-MPC2's planning with MPPI involves predicting 9216 latent states to attain similar performance levels. In contrast, EZ-V2's tree search-based planning only utilizes 32 imagined latent states, resulting in lighter computational demands.

% More specifically, MPPI utilizes 512 sampled trajectories over 3 steps in each planning. Besides, the computation in MPPI should be repeated multiple times to iteratively refine the parameters of a policy distribution. In contrast, EZ-V2 only uses 32 trajectories in the planning, which is more lightweight in computation. 
% EfficientZero V2 outperforms Dreamer V3 and SAC by 39\% and 35\% in the mean score, respectively. 
% Because TD-MPC2 outperforms TD-MPC in the DMControl benchmark, we do not report TD-MPC's results in Table \ref{tab:dmc_results_full}. 
% It is also worthy of mentioning that our method solves most tasks with fewer environment steps, which is illustrated in Fig. XX. 

\textbf{Vision Control}:
As shown in Table \ref{tab:dmc_results_full}, our method achieves a mean score of 726.1, surpassing the previous state-of-the-art, DreamerV3, by 45\%. Notably, it sets new records in 16 out of 20 tasks. Furthermore, our method demonstrates significant improvements in tasks with sparse rewards, as shown in Fig. \ref{dmc_vision}. For instance, in the `Cartpole-Swingup-Sparse' task, our method scores 763.6 compared to DreamerV3's 392.4. This substantial progress is attributed to two key algorithmic modifications: the planning with tree search, which ensures policy improvement and offers excellent exploratory capabilities, and the mixed value target, which enhances the accuracy of value learning, especially with stale data.

% In summary, EZ-V2 surpasses or closely matches the previous SOTA in each benchmark. 
As a general and sample-efficient RL framework, EZ-V2 consistently demonstrates high sample efficiency in tasks featuring low and high-dimensional observations, discrete and continuous action spaces, and both dense and sparse reward structures. Detailed training curves can be found in Appendix \ref{train_curves}.

\subsection{Ablation Study}
\label{ablation_sec}

In this section, we discuss the effectiveness of the main modifications: the sampling-based Gumbel search and the mixed value target. 

\textbf{Ablations of Search}: The comparative analysis between our search method and Sample MCTS is illustrated in Fig. \ref{ablation_search_main}. Sample MCTS, a tree search technique tailored for continuous control developed by Sample MuZero \citep{hubert2021learning}, is depicted in green. Our search method, highlighted in red, exhibits superior performance. Whereas Sample MCTS necessitates $n = 50$ simulations, our approach significantly reduces the computational burden to merely 32 simulations.

Further, we display performance curves for 16 and 8 simulations. As demonstrated in Fig. \ref{ablation_search_main}, an increased number of simulations enhances performance in complex tasks such as `Quadruped Walk', and notably, our method with merely 8 simulations surpasses Sample MCTS. This shows that the sampling-based Gumbel search achieves a superior balance between exploration and exploitation, backed by guaranteed policy improvement. Additional results for other tasks are provided in Appendix  \ref{ablation}.

\begin{figure}[t]
\centering
\includegraphics[width=0.45\textwidth]{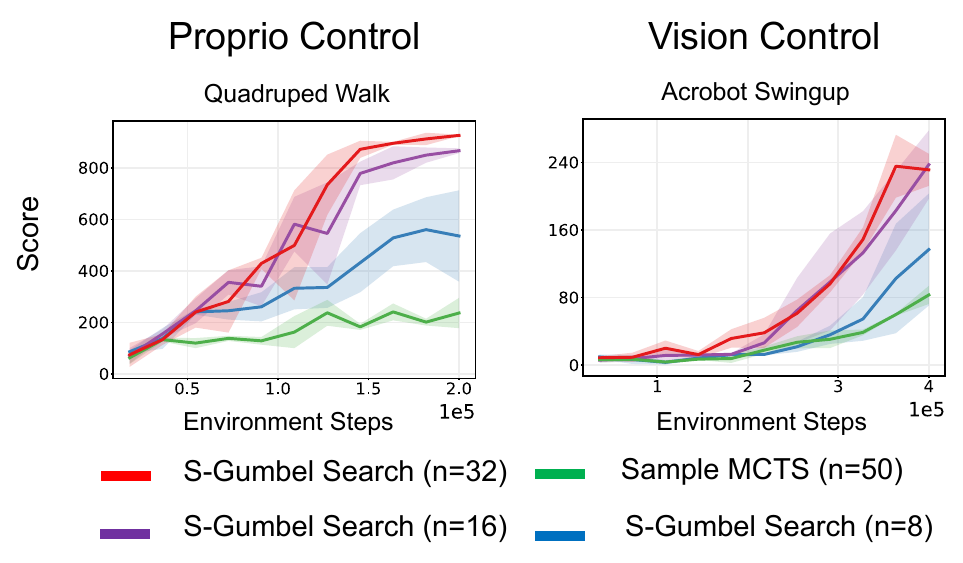}
\caption{
Ablation study of our search method, namely the sampling-based Gumbel search (S-Gumbel search). We compare it with our search method with different numbers of simulations (n=16, 8) and Sample MCTS \citep{hubert2021learning}. Our method outperforms Sample MCTS, and increasing the number of simulations improves our method's performance on hard tasks.}
\label{ablation_search_main}
\end{figure}

\begin{figure}[t]
\centering
\includegraphics[width=0.45\textwidth]{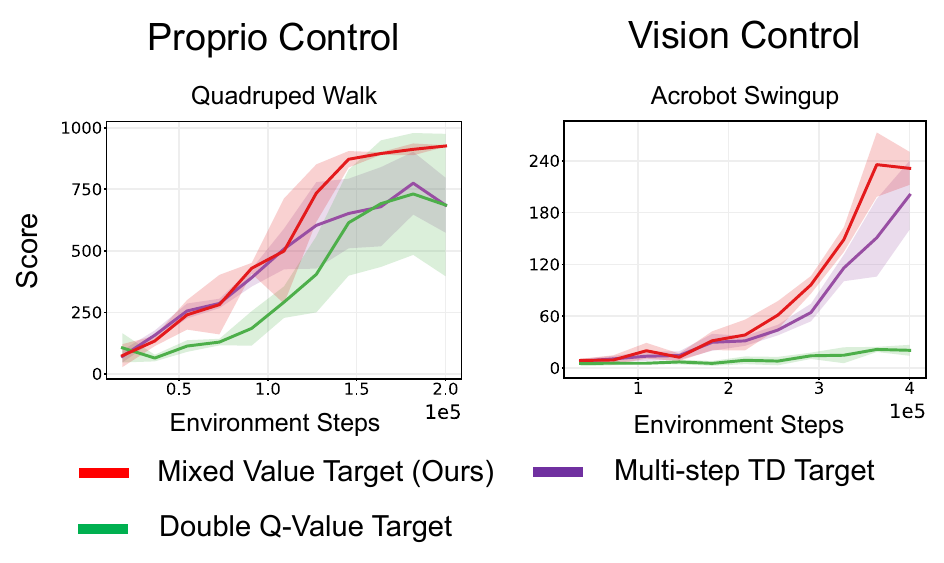}
\caption{
Ablation study of our value target, known as the mixed value target. We compare it with different value targets, including the multi-step TD target and the double Q-value target. The mixed value target consistently achieves high performance in both Proprio Control and Vision Control tasks.}
\label{ablation_value_main}
\vspace{-10pt}
\end{figure}
\vspace{-5pt}

\textbf{Ablations of Value Targets}. 
% The improved target value plays an important role in sample efficiency. 
% Since the benchmarks contain limited environment steps, it is essential to efficiently utilize the stale data from the replayed buffer. 
% Specifically, the value target is the sum of the predicted reward and the discount value of the predicted next state under the current policy. 
It can be seen in Fig. \ref{ablation_value_main} that our method (colored red) alleviates the off-policy issue compared with the multi-step TD target (colored purple), thus achieving better performance. Furthermore, we also compare with the value target 
$z_t$ derived from the optimal-Q Bellman equation, shown as follows:
\begin{equation}
    z_t = u_t + \gamma q(s_{t+1},a_{t+1}), \ \ a_{t+1} \sim p_t(a|s_{t+1})
\end{equation}
This technique also addresses the off-policy issue as the value target is based on the optimal Q-value estimation. Meanwhile, we employ the double Q-head trick. This estimation method is denoted as double Q-value target in Fig. \ref{ablation_value_main}.
Although the double Q-value target (colored green) also avoids the off-policy issue, the experiments illustrate that our method (colored red) exhibits consistent and robust performance across tasks. 
This is because our method matches the true value more rapidly by utilizing multi-step predicted rewards during the tree search.
Additional curves related to value ablation experiments can be found in Appendix \ref{ablation}.

Furthermore, practical comparisons between the TD-MPC2 and EZ-V2 algorithms in terms of computational load are provided in Appendix \ref{load}.

% Furthermore, EfficientZero V2 without the action embedding behaves with poor performance on some tasks, such as XX. The results indicate that the action embedding is moderately important for the learning of dynamic function due to the latent representation of the raw action. 

% \begin{figure*}[t]
% \centering
% \includegraphics[width=0.8\textwidth]{sections/figs/ablation_main.pdf}
% \caption{Comparison between EfficientZero V2 and Dreamer V3. \ywr{More baselines?} We evaluate them under the Atari 100k, Proprio Control 100k, and Vision Control 200k benchmarks. For each metric, we obtain the normalized score from all environments that satisfies the corresponding settings.}
% \label{exp}
% \end{figure*}

\section{Conclusion}
This paper presents EfficienctZero V2 (EZ-V2), a general framework for sample efficient RL. EZ-V2, which is built upon EfficientZero, extends to master continuous control. Furthermore, EZ-V2 achieves both superior performance and sample efficiency in tasks with visual and low-dimensional inputs. More specifically, EZ-V2 outperforms Dreamer V3 by a large margin across various types of domains, including Atari 100k, Proprio Control, and Vision Control benchmarks.  
% While EZ-V2 presents exciting possibilities, it also comes with certain limitations. 
Moreover, we evaluate the performance of EZ-V2 without conditions of limited data. EZ-V2 outperforms or is comparable to baselines with more interaction data, though the performance gap between EZ-V2 and other algorithms narrows as the amount of interaction data increases. 
Due to the high sample efficiency, we will broaden our evaluation of EZ-V2 across a wider variety of benchmarks, especially in real-world tasks like robotic manipulation. The superior sample efficiency of EZ-V2 holds great promise for enhancing online learning in real-world robotic scenarios.

% Future work should focus on enabling the learned model to handle safety considerations and risks, especially in real-world online learning scenarios like autonomous driving, where stochastic dynamics and interactions with other vehicles pose significant challenges. We encourage collaboration within the research community to tackle these challenges and drive further advancements in the field.

% While we are excited by the potential of EZ-V2, our limitations inspire future work. Taking online learning in the real world into account, we should enable the learned model to capture the unsafety regions and risks. In autonomous driving, the stochastic dynamics make the model learning more challenging due to the interaction with other vehicles.
% To address these challenges, pioneering research initiatives are essential, and we extend an invitation to the scholarly community to participate in these collaborative efforts.

\section*{Acknowledgements}

This work is supported by the Ministry of Science and Technology of the People's Republic of China, the 2030 Innovation Megaprojects ``Program on New Generation Artificial Intelligence'' (Grant No. 2021AAA0150000).  This work is also supported by the National Key R\&D Program of China (2022ZD0161700).

\section*{Impact Statement}

This paper presents work whose goal is to advance the field of 
Machine Learning. There are many potential societal consequences of our work, none which we feel must be specifically highlighted here.

\nocite{langley00}

\bibliography{example_paper}

\begin{thebibliography}{39}
\providecommand{\natexlab}[1]{#1}
\providecommand{\url}[1]{\texttt{#1}}
\expandafter\ifx\csname urlstyle\endcsname\relax
  \providecommand{\doi}[1]{doi: #1}\else
  \providecommand{\doi}{doi: \begingroup \urlstyle{rm}\Url}\fi

\bibitem[Akkaya et~al.(2019)Akkaya, Andrychowicz, Chociej, Litwin, McGrew, Petron, Paino, Plappert, Powell, Ribas, et~al.]{akkaya2019solving}
Akkaya, I., Andrychowicz, M., Chociej, M., Litwin, M., McGrew, B., Petron, A., Paino, A., Plappert, M., Powell, G., Ribas, R., et~al.
\newblock Solving rubik's cube with a robot hand.
\newblock \emph{arXiv preprint arXiv:1910.07113}, 2019.

\bibitem[Andrychowicz et~al.(2020)Andrychowicz, Baker, Chociej, Jozefowicz, McGrew, Pachocki, Petron, Plappert, Powell, Ray, et~al.]{andrychowicz2020learning}
Andrychowicz, O.~M., Baker, B., Chociej, M., Jozefowicz, R., McGrew, B., Pachocki, J., Petron, A., Plappert, M., Powell, G., Ray, A., et~al.
\newblock Learning dexterous in-hand manipulation.
\newblock \emph{The International Journal of Robotics Research}, 39\penalty0 (1):\penalty0 3--20, 2020.

\bibitem[Antonoglou et~al.(2021)Antonoglou, Schrittwieser, Ozair, Hubert, and Silver]{antonoglou2021planning}
Antonoglou, I., Schrittwieser, J., Ozair, S., Hubert, T.~K., and Silver, D.
\newblock Planning in stochastic environments with a learned model.
\newblock In \emph{International Conference on Learning Representations}, 2021.

\bibitem[Bellman(1957)]{bellman1957markovian}
Bellman, R.
\newblock A markovian decision process.
\newblock \emph{Journal of mathematics and mechanics}, pp.\  679--684, 1957.

\bibitem[Brockman et~al.(2016)Brockman, Cheung, Pettersson, Schneider, Schulman, Tang, and Zaremba]{1606.01540}
Brockman, G., Cheung, V., Pettersson, L., Schneider, J., Schulman, J., Tang, J., and Zaremba, W.
\newblock Openai gym, 2016.

\bibitem[Chen et~al.(2022)Chen, Xu, and Agrawal]{chen2022system}
Chen, T., Xu, J., and Agrawal, P.
\newblock A system for general in-hand object re-orientation.
\newblock In \emph{Conference on Robot Learning}, pp.\  297--307. PMLR, 2022.

\bibitem[Chen et~al.(2023)Chen, Tippur, Wu, Kumar, Adelson, and Agrawal]{chen2023visual}
Chen, T., Tippur, M., Wu, S., Kumar, V., Adelson, E., and Agrawal, P.
\newblock Visual dexterity: In-hand reorientation of novel and complex object shapes.
\newblock \emph{Science Robotics}, 8\penalty0 (84):\penalty0 eadc9244, 2023.
\newblock \doi{10.1126/scirobotics.adc9244}.
\newblock URL \url{https://www.science.org/doi/abs/10.1126/scirobotics.adc9244}.

\bibitem[Chen \& He(2021)Chen and He]{chen2021exploring}
Chen, X. and He, K.
\newblock Exploring simple siamese representation learning.
\newblock In \emph{Proceedings of the IEEE/CVF conference on computer vision and pattern recognition}, pp.\  15750--15758, 2021.

\bibitem[Coulom(2006)]{coulom2006efficient}
Coulom, R.
\newblock Efficient selectivity and backup operators in monte-carlo tree search.
\newblock In \emph{International conference on computers and games}, pp.\  72--83. Springer, 2006.

\bibitem[Danihelka et~al.(2021)Danihelka, Guez, Schrittwieser, and Silver]{danihelka2021policy}
Danihelka, I., Guez, A., Schrittwieser, J., and Silver, D.
\newblock Policy improvement by planning with gumbel.
\newblock In \emph{International Conference on Learning Representations}, 2021.

\bibitem[De~Asis et~al.(2018)De~Asis, Hernandez-Garcia, Holland, and Sutton]{de2018multi}
De~Asis, K., Hernandez-Garcia, J., Holland, G., and Sutton, R.
\newblock Multi-step reinforcement learning: A unifying algorithm.
\newblock In \emph{Proceedings of the AAAI Conference on Artificial Intelligence}, volume~32, 2018.

\bibitem[Feinberg et~al.(2018)Feinberg, Wan, Stoica, Jordan, Gonzalez, and Levine]{feinberg2018model}
Feinberg, V., Wan, A., Stoica, I., Jordan, M.~I., Gonzalez, J.~E., and Levine, S.
\newblock Model-based value expansion for efficient model-free reinforcement learning.
\newblock In \emph{Proceedings of the 35th International Conference on Machine Learning (ICML 2018)}, 2018.

\bibitem[Haarnoja et~al.(2018)Haarnoja, Zhou, Abbeel, and Levine]{haarnoja2018soft}
Haarnoja, T., Zhou, A., Abbeel, P., and Levine, S.
\newblock Soft actor-critic: Off-policy maximum entropy deep reinforcement learning with a stochastic actor.
\newblock In \emph{International conference on machine learning}, pp.\  1861--1870. PMLR, 2018.

\bibitem[Hafner et~al.(2019)Hafner, Lillicrap, Ba, and Norouzi]{hafner2019dream}
Hafner, D., Lillicrap, T., Ba, J., and Norouzi, M.
\newblock Dream to control: Learning behaviors by latent imagination.
\newblock \emph{arXiv preprint arXiv:1912.01603}, 2019.

\bibitem[Hafner et~al.(2023)Hafner, Pasukonis, Ba, and Lillicrap]{hafner2023mastering}
Hafner, D., Pasukonis, J., Ba, J., and Lillicrap, T.
\newblock Mastering diverse domains through world models.
\newblock \emph{arXiv preprint arXiv:2301.04104}, 2023.

\bibitem[Hansen et~al.(2022)Hansen, Wang, and Su]{hansen2022temporal}
Hansen, N., Wang, X., and Su, H.
\newblock Temporal difference learning for model predictive control.
\newblock \emph{arXiv preprint arXiv:2203.04955}, 2022.

\bibitem[Hansen et~al.(2023)Hansen, Su, and Wang]{Anonymous2023TDMPC2}
Hansen, N., Su, H., and Wang, X.
\newblock Td-mpc2: Scalable, robust world models for continuous control, 2023.

\bibitem[Hubert et~al.(2021)Hubert, Schrittwieser, Antonoglou, Barekatain, Schmitt, and Silver]{hubert2021learning}
Hubert, T., Schrittwieser, J., Antonoglou, I., Barekatain, M., Schmitt, S., and Silver, D.
\newblock Learning and planning in complex action spaces.
\newblock In \emph{International Conference on Machine Learning}, pp.\  4476--4486. PMLR, 2021.

\bibitem[Hwangbo et~al.(2019)Hwangbo, Lee, Dosovitskiy, Bellicoso, Tsounis, Koltun, and Hutter]{hwangbo2019learning}
Hwangbo, J., Lee, J., Dosovitskiy, A., Bellicoso, D., Tsounis, V., Koltun, V., and Hutter, M.
\newblock Learning agile and dynamic motor skills for legged robots.
\newblock \emph{Science Robotics}, 4\penalty0 (26):\penalty0 eaau5872, 2019.

\bibitem[Kaiser et~al.(2019)Kaiser, Babaeizadeh, Milos, Osinski, Campbell, Czechowski, Erhan, Finn, Kozakowski, Levine, et~al.]{kaiser2019model}
Kaiser, L., Babaeizadeh, M., Milos, P., Osinski, B., Campbell, R.~H., Czechowski, K., Erhan, D., Finn, C., Kozakowski, P., Levine, S., et~al.
\newblock Model-based reinforcement learning for atari.
\newblock \emph{arXiv preprint arXiv:1903.00374}, 2019.

\bibitem[Karnin et~al.(2013)Karnin, Koren, and Somekh]{karnin2013almost}
Karnin, Z., Koren, T., and Somekh, O.
\newblock Almost optimal exploration in multi-armed bandits.
\newblock In \emph{International conference on machine learning}, pp.\  1238--1246. PMLR, 2013.

\bibitem[Kool et~al.(2019)Kool, Van~Hoof, and Welling]{kool2019stochastic}
Kool, W., Van~Hoof, H., and Welling, M.
\newblock Stochastic beams and where to find them: The gumbel-top-k trick for sampling sequences without replacement.
\newblock In \emph{International Conference on Machine Learning}, pp.\  3499--3508. PMLR, 2019.

\bibitem[Laskin et~al.(2020)Laskin, Srinivas, and Abbeel]{laskin2020curl}
Laskin, M., Srinivas, A., and Abbeel, P.
\newblock Curl: Contrastive unsupervised representations for reinforcement learning.
\newblock In \emph{International Conference on Machine Learning}, pp.\  5639--5650. PMLR, 2020.

\bibitem[Petrenko et~al.(2023)Petrenko, Allshire, State, Handa, and Makoviychuk]{petrenko2023dexpbt}
Petrenko, A., Allshire, A., State, G., Handa, A., and Makoviychuk, V.
\newblock Dexpbt: Scaling up dexterous manipulation for hand-arm systems with population based training.
\newblock \emph{arXiv preprint arXiv:2305.12127}, 2023.

\bibitem[Rubinstein(1997)]{rubinstein1997optimization}
Rubinstein, R.~Y.
\newblock Optimization of computer simulation models with rare events.
\newblock \emph{European Journal of Operational Research}, 99\penalty0 (1):\penalty0 89--112, 1997.

\bibitem[Schrittwieser et~al.(2020)Schrittwieser, Antonoglou, Hubert, Simonyan, Sifre, Schmitt, Guez, Lockhart, Hassabis, Graepel, et~al.]{schrittwieser2020mastering}
Schrittwieser, J., Antonoglou, I., Hubert, T., Simonyan, K., Sifre, L., Schmitt, S., Guez, A., Lockhart, E., Hassabis, D., Graepel, T., et~al.
\newblock Mastering atari, go, chess and shogi by planning with a learned model.
\newblock \emph{Nature}, 588\penalty0 (7839):\penalty0 604--609, 2020.

\bibitem[Schrittwieser et~al.(2021)Schrittwieser, Hubert, Mandhane, Barekatain, Antonoglou, and Silver]{schrittwieser2021online}
Schrittwieser, J., Hubert, T., Mandhane, A., Barekatain, M., Antonoglou, I., and Silver, D.
\newblock Online and offline reinforcement learning by planning with a learned model.
\newblock \emph{Advances in Neural Information Processing Systems}, 34:\penalty0 27580--27591, 2021.

\bibitem[Schwarzer et~al.(2020)Schwarzer, Anand, Goel, Hjelm, Courville, and Bachman]{schwarzer2020data}
Schwarzer, M., Anand, A., Goel, R., Hjelm, R.~D., Courville, A., and Bachman, P.
\newblock Data-efficient reinforcement learning with self-predictive representations.
\newblock \emph{arXiv preprint arXiv:2007.05929}, 2020.

\bibitem[Schwarzer et~al.(2023)Schwarzer, Ceron, Courville, Bellemare, Agarwal, and Castro]{schwarzer2023bigger}
Schwarzer, M., Ceron, J. S.~O., Courville, A., Bellemare, M.~G., Agarwal, R., and Castro, P.~S.
\newblock Bigger, better, faster: Human-level atari with human-level efficiency.
\newblock In \emph{International Conference on Machine Learning}, pp.\  30365--30380. PMLR, 2023.

\bibitem[Silver et~al.(2016)Silver, Huang, Maddison, Guez, Sifre, Van Den~Driessche, Schrittwieser, Antonoglou, Panneershelvam, Lanctot, et~al.]{silver2016mastering}
Silver, D., Huang, A., Maddison, C.~J., Guez, A., Sifre, L., Van Den~Driessche, G., Schrittwieser, J., Antonoglou, I., Panneershelvam, V., Lanctot, M., et~al.
\newblock Mastering the game of go with deep neural networks and tree search.
\newblock \emph{nature}, 529\penalty0 (7587):\penalty0 484--489, 2016.

\bibitem[Silver et~al.(2017)Silver, Hubert, Schrittwieser, Antonoglou, Lai, Guez, Lanctot, Sifre, Kumaran, Graepel, et~al.]{silver2017masteringchess}
Silver, D., Hubert, T., Schrittwieser, J., Antonoglou, I., Lai, M., Guez, A., Lanctot, M., Sifre, L., Kumaran, D., Graepel, T., et~al.
\newblock Mastering chess and shogi by self-play with a general reinforcement learning algorithm.
\newblock \emph{arXiv preprint arXiv:1712.01815}, 2017.

\bibitem[Silver et~al.(2018)Silver, Hubert, Schrittwieser, Antonoglou, Lai, Guez, Lanctot, Sifre, Kumaran, Graepel, et~al.]{silver2018general}
Silver, D., Hubert, T., Schrittwieser, J., Antonoglou, I., Lai, M., Guez, A., Lanctot, M., Sifre, L., Kumaran, D., Graepel, T., et~al.
\newblock A general reinforcement learning algorithm that masters chess, shogi, and go through self-play.
\newblock \emph{Science}, 362\penalty0 (6419):\penalty0 1140--1144, 2018.

\bibitem[Tassa et~al.(2018)Tassa, Doron, Muldal, Erez, Li, Casas, Budden, Abdolmaleki, Merel, Lefrancq, et~al.]{tassa2018deepmind}
Tassa, Y., Doron, Y., Muldal, A., Erez, T., Li, Y., Casas, D. d.~L., Budden, D., Abdolmaleki, A., Merel, J., Lefrancq, A., et~al.
\newblock Deepmind control suite.
\newblock \emph{arXiv preprint arXiv:1801.00690}, 2018.

\bibitem[Wu et~al.(2022)Wu, Escontrela, Hafner, Goldberg, and Abbeel]{daydreamer}
Wu, P., Escontrela, A., Hafner, D., Goldberg, K., and Abbeel, P.
\newblock Daydreamer: World models for physical robot learning, 2022.

\bibitem[Xiong et~al.(2020)Xiong, Yang, He, Zheng, Zheng, Xing, Zhang, Lan, Wang, and Liu]{xiong2020layer}
Xiong, R., Yang, Y., He, D., Zheng, K., Zheng, S., Xing, C., Zhang, H., Lan, Y., Wang, L., and Liu, T.
\newblock On layer normalization in the transformer architecture.
\newblock In \emph{International Conference on Machine Learning}, pp.\  10524--10533. PMLR, 2020.

\bibitem[Yarats et~al.(2021)Yarats, Fergus, Lazaric, and Pinto]{yarats2021mastering}
Yarats, D., Fergus, R., Lazaric, A., and Pinto, L.
\newblock Mastering visual continuous control: Improved data-augmented reinforcement learning.
\newblock \emph{arXiv preprint arXiv:2107.09645}, 2021.

\bibitem[Ye et~al.(2021)Ye, Liu, Kurutach, Abbeel, and Gao]{ye2021mastering}
Ye, W., Liu, S., Kurutach, T., Abbeel, P., and Gao, Y.
\newblock Mastering atari games with limited data.
\newblock \emph{Advances in Neural Information Processing Systems}, 34:\penalty0 25476--25488, 2021.

\bibitem[Zhao et~al.(2022)Zhao, Tu, and Xu]{zhao2022efficient}
Zhao, D., Tu, S., and Xu, L.
\newblock Efficient learning for alphazero via path consistency.
\newblock In \emph{International Conference on Machine Learning}, pp.\  26971--26981. PMLR, 2022.

\bibitem[Zheng et~al.(2015)Zheng, Yang, Liu, Liang, and Li]{zheng2015improving}
Zheng, H., Yang, Z., Liu, W., Liang, J., and Li, Y.
\newblock Improving deep neural networks using softplus units.
\newblock In \emph{2015 International joint conference on neural networks (IJCNN)}, pp.\  1--4. IEEE, 2015.

\end{thebibliography}
\bibliographystyle{icml2023}

%%%%%%%%%%%%%%%%%%%%%%%%%%%%%%%%%%%%%%%%%%%%%%%%%%%%%%%%%%%%%%%%%%%%%%%%%%%%%%%
%%%%%%%%%%%%%%%%%%%%%%%%%%%%%%%%%%%%%%%%%%%%%%%%%%%%%%%%%%%%%%%%%%%%%%%%%%%%%%%
% APPENDIX
%%%%%%%%%%%%%%%%%%%%%%%%%%%%%%%%%%%%%%%%%%%%%%%%%%%%%%%%%%%%%%%%%%%%%%%%%%%%%%%
%%%%%%%%%%%%%%%%%%%%%%%%%%%%%%%%%%%%%%%%%%%%%%%%%%%%%%%%%%%%%%%%%%%%%%%%%%%%%%%
\newpage
\onecolumn
\appendix
\section{Summary of Differences}
\label{summary_diff}
EfficientZero V2 builds upon EffcientZero algorithm \citep{ye2021mastering}. This section demonstrates the major applied changes to achieve mastering performance across domains.
\begin{itemize}
    \item \textbf{Search}: Different from the MCTS in EfficientZero, we employ Gumbel search which differs in action selections. Gumbel search guarantees policy improvement even if with limited simulation budgets, which significantly reduces the computation of EZ-V2.
    \item \textbf{Search-based Value Estimation}: Compared to the adaptive TD method used in EfficientZero's value estimation, we employ the empirical mean of the search root node as the target value. The search process utilizes the current model and policy to calculate improved policy and value estimation, which can improve the utilization of early-stage transitions.
    \item \textbf{Gaussian Policy}: Inspired by Sampled MuZero \citep{hubert2021learning}, we employ an Gaussian distribution, which is parameterized by the learnable policy function, to represent the policy in continuous action spaces. We generate search action candidates by simply sampling from the Gaussian policy, which naturally satisfies the sampling without replacement in Gumbel search. We then prove the policy improvement of Gumbel search still holds in the continuous setting.
    \item \textbf{Action Embedding}: We employ an action embedding layer to encode the real actions as latent vectors. By representing actions in a hidden space, actions that are similar to each other are placed closer in the embedding space. This proximity allows the RL agent to generalize its learning from one action to similar actions, improving its efficiency and performance.
    \item \textbf{Priority Precalculation}: Conventionally, the priorities of a newly collected trajectory are set to be the maximal priority of total collected transitions. We propose to warm up the new priorities by calculating the bellman error using the current model. This increases the probability of newly collected transitions being replayed, thereby improving sample efficiency.
    \item \textbf{Architecture}:
    For the 2-dimensional image inputs, we follow the most implementation of the architecture of EfficientZero. For the continuous control task with 1-dimensional input, we use a variation of the previous architecture in which all convolutions are replaced by fully connected layers. The details could be found in Appendix \ref{arch}.
    \item \textbf{Hyperparameters}: We tuned the hyperparameters of EfficientZero V2 to achieve satisfying performance across various domains. The generality is verified by training from scratch without further adjustments, including Atari 100k, Proprio Control, Vision Control. More details refer to Appendix \ref{hyper-param}.
\end{itemize}

\section{Calculation of Target Policy}
\label{cal_ip}

In discrete control, the calculation of the target policy is the same as that of the original Gumbel search, which can be found in Gumbel Muzero \citep{danihelka2021policy}.
In a continuous setting, we modify the calculation of target policy as follows.
\begin{equation}
    \pi^{\prime}=\text{softmax}(\sigma(\text{completedQ}))
\end{equation}
where completedQ is a comprehensive value estimation of visited and unvisited candidates. For visited nodes, the completedQ is calculated via the empirical estimation as $q(a)=r(s,a)+\gamma v(s^\prime)$, $v(s^\prime)$ is the empirical mean of bootstrapped value sum and visit counts. For non-visits, $\text{completedQ}$ is estimated through the weighted average Q-value of visited nodes as $\sum_{a}\pi(a)q(a)$.  

We mention the transformation $\sigma$ is a monotonically increasing transformation in the main text. For a concrete instantiation of $\sigma$, we use
\begin{equation}
\sigma(q(s,a))=\left(c_{\text {visit }}+\max _b N(b)\right) c_{\text {scale }}q(s,a)
\end{equation}
where $\max _b N(b)$ is the visit count of the most visited action. $c_{\text {visit }}$ and $c_{\text {scale }}$ are 50 and 0.1 respectively.

\section{Advantage of Simple Policy Loss}
\label{simplepi}

Different from minimizing the cross-entropy between the policy with the target policy, the simple loss using $a^*_S$ directly improves the possibility of the recommended action in tree search. That means we can make the policy network output a good action in the early stage. With the iteration repeating, the policy network can find the optimal point. More specifically, we provide an intuitive example to illustrate that the policy can reach the optimal point more efficiently, as shown in Fig. \ref{appb:simplepi}. The whole space represents the action space. The simple loss enforces the policy to output the current best point colored red. The original cross-entropy loss considers all actions and thus makes the policy's output near the point colored brown. We can see that when the action space is large, the red point can guide the policy to reach the optimal point colored purple more quickly.

\begin{figure}[h]
\centering
\includegraphics[width=0.5\textwidth]{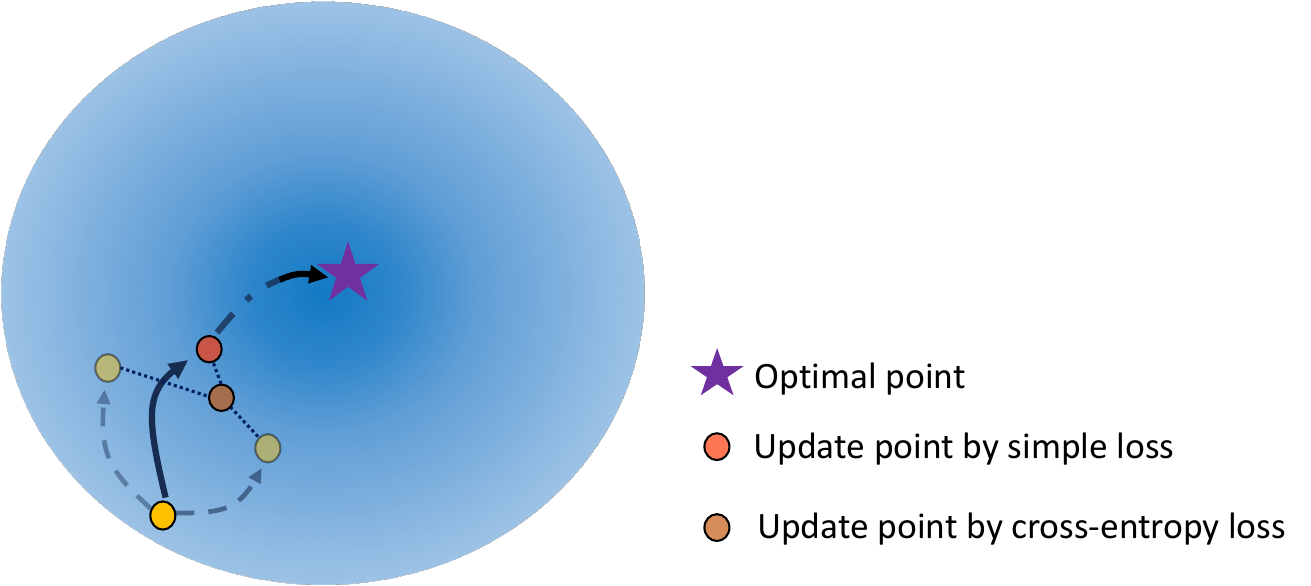}
\caption{Intuitive example showing the difference between simple policy loss using $a^*_S$ and cross-entropy loss.}
\label{appb:simplepi}
\end{figure}

\section{Details of Search-based Value Estimation (SVE)}
\subsection{Calculation of SVE}
\label{cal_sve}
First, the search process expands a tree through $N$ simulations gradually. At each simulation, the agent dives to a leaf and expands a new child. The diving path is easily associated with an $H(n)$-step rollout, which forms an imagined H(n)-step value estimation for the root node. This estimation process will be repeated $N$ times, to get an average estimation, as defined in Definition \ref{app:def:sve}.

\begin{definition}[\textbf{Search-Based Value Estimation}]
\label{app:def:sve}
\textit{Using imagined states and rewards $\hat{s}_{t+1}, \hat{r}_t=\mathcal{G}(\hat{s}_t,\hat{a}_t)$ obtained from our learnable dynamic function, 
% the search-based value estimation of a given state $s_0$ is defined as
the value estimation of a given state $s_0$ can be derived from the empirical mean of $N$ bootstrapped estimations, which is formulated as
}
\begin{equation}
    \hat{V}_\text{S}(s_0)=\frac{\sum_{n=0}^{N}\hat{V}_n(s_0)}{N}
\end{equation}
\textit{where $N$ denotes the number of simulations, $\hat{V}_n(s_0)$ is the bootstrapped estimation of the $n$-th node expansion, which is formulated as }
\begin{equation}            \hat{V}_n(s_0)=\sum_{t=0}^{H(n)}\gamma^t\hat{r}_t+\gamma^{H(n)}\hat{V}(\hat{s}_{H(n)})
\end{equation}
\textit{where $H(n)$ denotes the search depth of the $n$-th iteration.}
\end{definition}

\section{Proof for SVE Error Bound}
\begin{corollary}[\textbf{Search-Based Value Estimation Error}]
\label{app:theorem:sve_error}
    Define $s_t,a_t,r_t$ to be the states, actions, and rewards resulting from current policy $\pi$ using true dynamics $\mathcal{G}^*$ and reward function $\mathcal{R}^*$, starting from $s_0\sim\nu$ and similarly define $\hat{s}_t, \hat{a}_t, \hat{r_t}$ using learned function $\mathcal{G}$. Let reward function $\mathcal{R}$ to be $L_r-Lipschitz$ and value function $\mathcal{V}$ as $L_V-Lipschitz$. Assume $\epsilon_s, \epsilon_r, \epsilon_v$ as upper bounds of state transition, reward, and value estimations respectively. We define the error bounds of each estimation as
    % \begin{equation}
    %     \max_{n\in[N],t\in[H(n)]}\mathbb{E}\left[\Vert\hat{s}_t-s_t\Vert^2\right]\leq\epsilon^2
    % \end{equation}
    \begin{gather}
        \max_{n\in[N],t\in[H(n)]}\mathbb{E}\left[\Vert\hat{s}_t-s_t\Vert^2\right]\leq\epsilon_s^2 \\
        \max_{n\in[N],t\in[H(n)]}\mathbb{E}\left[\Vert\mathcal{R}(s_t)-\mathcal{R}^*(s_t)\Vert^2\right]\leq\epsilon_r^2 \\
        \max_{n\in[N],t\in[H(n)]}\mathbb{E}\left[\Vert\mathcal{V}(s_t)-\mathcal{V}^*(s_t)\Vert^2\right]\leq\epsilon_v^2
    \end{gather}
    within a tree-search process. Then we have errors
    \begin{equation}
    \begin{aligned}
        \text{MSE}_\nu(\hat{V}_\text{S})\leq\frac{4}{N^2}\sum_{n=0}^N\left(\sum_{t=0}^{H(n)}\gamma^{2t}(L_r^2\epsilon_s^2+\epsilon_r^2)+\gamma^{2H(n)}(L_V^2\epsilon_s^2+\epsilon_v^2)\right)
    \end{aligned}
    \end{equation}
    % Notably, the coefficient of $\epsilon^2$ 
    % \begin{equation*}
    %     \frac{2}{N^2}\sum_{n=0}^N\left(\sum_{t=0}^{H(n)}\gamma^{2t}L_r^2+\gamma^{2H(n)}L_V^2\right)
    % \end{equation*}
    % is convergent with $H(n)$. Additionally, 
    where $N$ is the simulation number of the search process. $H(n)$ denotes the depth of the $n$-th search iteration.
    % The upper bound will converge to 0 when the dynamic function is approximately optimal $\epsilon\to 0$.
\end{corollary}

\begin{proof}
To provide detailed proof for the upper bound of the MSE of the search-based value estimation $\hat{V}_\text{S}$, we follow a structured approach inspired by Model-based Value Estimation (MVE) \citep{feinberg2018model}, with adjustments for the specifics of MCTS and considerations of errors of reward and value estimations.

Given Definition \ref{def:sve}, we aim to bound the MSE of this estimator, defined as:
\begin{equation}
    \text{MSE}_\nu(\hat{V}_\text{S})=\mathbb{E}\left[\left(\hat{V}_\text{S}(s_0)-V_\pi(s_0)\right)^2\right]
\end{equation}
Where $V_\pi(s_0)=\sum_{t=0}^{H-1}\gamma^{t}r_t+\gamma^{H}V_\pi(s_H)$. We can first decompose the MSE as

\begin{equation}
\begin{split}
\label{e1}
    \text{MSE}_\nu(\hat{V}_\text{S})=\mathbb{E}\Bigg[ \Bigg(\frac{1}{N} \sum_{n=0}^N\Bigg(\sum_{t=0}^{H(n)}\gamma^t(\hat{r}_t-r_t) +
    \gamma^{H(n)}\left(\hat{V}(\hat{s}_{H(n)})-V_\pi(s_{H(n)})\right)\Bigg)\Bigg)^2 \Bigg]
\end{split}
\end{equation}

According to $L-Lipschitz$ continuity of $\mathcal{R}$ and Cauchy inequality, we can derive
\begin{equation*}
\begin{aligned}
    \mathbb{E}\left[(\hat{r}_t-r_t)^2\right]&=\mathbb{E}\left[(\mathcal{R}(\hat{s}_t)-\mathcal{R}^*(s_t))^2\right]\\
    &=\mathbb{E}\left[(\mathcal{R}(\hat{s}_t)-\mathcal{R}(s_t)+\mathcal{R}(s_t)-\mathcal{R}^*(s_t))^2\right]\\
    &\leq 2\mathbb{E}\left[(\mathcal{R}(\hat{s}_t)-\mathcal{R}(s_t))^2\right]+2\mathbb{E}\left[(\mathcal{R}(s_t)-\mathcal{R}^*(s_t))^2\right] (\text{Cauchy inequality})\\
    &\leq 2L_r^2\mathbb{E}\left[\Vert\hat{s}_t-s_t\Vert^2\right]+2\mathbb{E}\left[\Vert\mathcal{R}(s_t)-\mathcal{R}^*(s_t)\Vert^2\right]\\
    &\leq 2L_r^2\epsilon_s^2+2\epsilon_r^2
\end{aligned}
\end{equation*}
% \begin{equation*}
% \begin{aligned}
%     \mathbb{E}\left[(\hat{r}_t-r_t)^2\right]&\leq L_r^2\mathbb{E}\left[\Vert\hat{s}_t-s_t\Vert^2\right]\\
%     % \mathbb{E}\left[\left(\hat{V}(\hat{s}_{H(n)})-V_\pi(s_{H(n)})\right)^2\right]&\leq L_V^2\mathbb{E}\left[\Vert\hat{s}_{H(n)}-s_{H(n)}\Vert^2\right]
%     \mathbb{E}\left[\left(\hat{V}(\hat{s}_{t})-V_\pi(s_{t})\right)^2\right]&\leq L_V^2\mathbb{E}\left[\Vert\hat{s}_{t}-s_{t}\Vert^2\right]
% \end{aligned}
% \end{equation*}
Similarly, we derive the error bound of value estimation as
\begin{equation*}
    \mathbb{E}\left[\left(\hat{V}(\hat{s}_{t})-V_\pi(s_{t})\right)^2\right]=\mathbb{E}\left[(\mathcal{V}(\hat{s}_t)-\mathcal{V}^*(s_t))^2\right]\leq 2L_v^2\epsilon_s^2+2\epsilon_v^2
\end{equation*}

Assume the model inference errors are additive per step, hence
% at each step are \textit{iid}, then the total error over $h$ steps could be though of a random walk, hence
% \begin{equation*}
% \begin{aligned}
%     \mathbb{E}\left[(\hat{r}_{t+h}-r_{t+h})^2\right]&\leq  L_r^2\mathbb{E}\left[\Vert\hat{s}_{t+h}-s_{t+h}\Vert^2\right]\\
%     \mathbb{E}\left[\left(\hat{V}(\hat{s}_{t+h})-V_\pi(s_{t+h})\right)^2\right]&\leq  L_V^2\mathbb{E}\left[\Vert\hat{s}_{t+h}-s_{t+h}\Vert^2\right]
% \end{aligned}
% \end{equation*}

% Using Cauchy-Schewartz inequality $(a+b)^2\leq2a^2+2b^2$, we have
\begin{equation}
\begin{split}
    \text{MSE}_\nu(\hat{V}_\text{S})\leq& \frac{2}{N^2}\sum_{n=0}^N\Bigg(\sum_{t=0}^{H(n)}\gamma^{2t}(\hat{r}_t-r_t)^2+
    \gamma^{2H(n)}\left(\hat{V}(\hat{s}_{H(n)})-V_\pi(s_{H(n)})\right)^2\Bigg)\\
\leq&
% \frac{2\epsilon^2}{N^2}\sum_{n=0}^N\left(\sum_{t=0}^{H(n)}\gamma^{2t}L_r^2+\gamma^{2H(n)}L_V^2\right)
\frac{4}{N^2}\sum_{n=0}^N\left(\sum_{t=0}^{H(n)}\gamma^{2t}(L_r^2\epsilon_s^2+\epsilon_r^2)+\gamma^{2H(n)}(L_V^2\epsilon_s^2+\epsilon_v^2)\right)
\end{split}
\end{equation}
Considering the convergence
% whether the coefficient part 
% \begin{equation}
%     \frac{2}{N^2}\sum_{n=0}^N\left(\sum_{t=0}^{H(n)}\gamma^{2t}L_r^2+\gamma^{2H(n)}L_V^2\right)
% \end{equation}
% converges 
with increasing search depth $H(n)$, we separate it into two parts:
\begin{enumerate}
    \item \textbf{The reward error term} $\sum_{t=0}^{H(n)}\gamma^{2t}(L_r^2\epsilon_s^2+\epsilon_r^2)$:
    Given $L_r\in \mathcal{C}$, $\gamma\in(0,1)$, and $\epsilon_s,\epsilon_r$ decaying with training, the series is obviously convergent.
    % let $M(H)=\sum_{t=0}^H t\gamma^{2t}$, we have
    % \begin{equation*}
    %     M(H)-M(H-1)=H\gamma^{2H}
    % \end{equation*}
    % We need to prove $\lim_{H\to\infty}H\gamma^{2H}=0$. Using ratio method, we have
    % \begin{equation*}
    % \begin{aligned}
    %     \frac{(H+1)\gamma^{2(H+1)}}{H\gamma^{2H}}=(1+\frac{1}{H})\gamma^2<1
    % \end{aligned}
    % \end{equation*}
    \item \textbf{The terminal state value error term} $\gamma^{2H(n)}(L_V^2\epsilon_s^2+\epsilon_v^2)$ is similarly converged. 
    % to 0 if $H(n)\to\infty$.
\end{enumerate}
On the other hand, this upper bound will also converge to 0 if the model is optimal $\epsilon_s,\epsilon_r,\epsilon_v\to0$.
\end{proof}

Compared to SVE, the estimation error of multi-step TD methods in \citep{de2018multi} depends on sampling, making it difficult to determine their theoretical upper bound of estimation errors. In other words, the estimation error of multi-step TD methods increases gradually during learning due to the off-policy issue.

\section{Details of Mixed Value Target}
\label{mixedV}
The mixed value target is calculated for each transition sampled from the buffer. It contains two types of value targets: the search-based value target and the multi-step TD target.

Specifically, we use two criteria to determine if we should use a search-based value target. The first criterion is whether the sampled transition comes from recent rollouts. If the transition is not from recent rollouts, meaning the transition is considered stale, we opt for the search-based value as the target. Otherwise, the multi-step TD (Temporal Difference) target is employed. The second criterion considers whether the current training step is in the early stage. During this phase, all transitions in the buffer are fresh. Additionally, the error in the dynamic model is still significant, resulting in inaccuracies in search-based value estimation. Therefore, we opt for the multi-step TD target as the target value.

\section{Details of Achitecture}
\label{arch}
For the 2-dimensional image inputs, we follow the most implementation of the architecture of EfficientZero. For the continuous control task with 1-dimensional input, we use a variation of the previous architecture in which all convolutions are replaced by fully connected layers. In the following, we describe the detailed architecture of EZ-V2 under 1-dimensional input.

The representation function $\mathcal{H}$ first processes the observation by a running mean block. The running mean block is similar to a Batch Normalization layer without learnable parameters. Then, the normalized input is processed by a linear layer, followed by a Layer Normalisation and a Tanh activation. Hereafter, we use a Pre-LN Transformer style pre-activation residual tower \citep{xiong2020layer} coupled with Layer Normalisation and Rectified Linear Unit (ReLU) activations to obtain the latent state. We used 3 blocks and the output dim is 128. Each linear layer has a hidden size of 256.

The dynamic function $\mathcal{G}$ takes the state and the action embedding as inputs. the action embedding is obtained from the action embedding layer which consists of a linear layer, a Layer Normalization, and a ReLU activation. The size of the action embedding is 64. The combination of the state and the action embedding are also processed by Pre-LN Transformer style pre-activation residual tower \citep{xiong2020layer} coupled with Layer Normalisation and ReLU activations. 

For the reward $\mathcal{R}$, value $\mathcal{V}$ and policy $\mathcal{P}$ function share the similar network structures. Taking the state as input, a linear layer followed by a Layer Normalization obtains the hidden variables. Then, we use the MLP network combined with Batch Normalization, which is similar to that of EfficientZero, to obtain the reward, value, and policy prediction. The hidden size of each layer is 256, and the activation function is ReLU. 

The reward and value predictions used the categorical representation introduced in EfficientZero. We used 51 bins for both the value and the reward predictions with the value being able to represent values between $\{-299.0, 299.0\}$. The reward can represent values between $\{-2.0, 2.0\}$. The maximum reward is 2 because the action repeat in DMControl is 2.
For policy function, the network outputs the mean and standard deviation of the Gaussian distribution. Then we use the 5 times Tanh function to restrict the range of the mean, and Softplus \citep{zheng2015improving} function to make the standard deviation over 0. In addition, the policy distribution is modeled by a squashed Gaussian distribution. A squashed Gaussian distribution belongs to a modification of a standard Gaussian, where the outputs are transformed into a bounded interval. 

Furthermore, we add a running mean block for observation in the representation function $\mathcal{H}$ in continuous control whose observation is 1 dimension. A key benefit of the module is that it normalizes the observation to mitigate exploding gradients.

\section{Training pipeline}
\label{pipeline}

The training pipeline comprises data workers, batch workers, and a learner. The data workers, also known as self-play workers, collect trajectories based on the model updated at specific intervals. The actions executed in these trajectories are determined using the sampling-based Gumbel search, as depicted in Figure \ref{framework} (B).

On the other hand, batch workers provide batch transitions sampled from the replay buffer to the learner. Similar to EfficientZero, the target policy and value in batch transitions are reanalyzed with the latest target model. This reanalysis involves revisiting past trajectories and re-executing the data using the target model, resulting in fresher search-based values and target policies obtained through model inference and Gumbel search.
The target model undergoes periodic updates at specified intervals during the training process. 

Finally, the learner trains the reward, dynamics, value, and policy functions using the reanalyzed batch. Figure \ref{framework} (A) and Equation \eqref{loss} illustrate the specific losses involved in the training process. To enhance the learning process, we have designed a parallel training framework where data workers, batch workers, and a learner all operate concurrently.

\section{Hyperparameters of Algorithms}
\label{hyper-param}

\subsection{Hyperparameters of Our Method}

We employ similar hyperparameters across all domains, as outlined in Table \ref{tab:hparams}. It's worth noting that we use different optimizers in tasks with different inputs. Specifically, due to architectural differences, we opt for the Adam optimizer in the 'Proprio Control 50-100k' task, whereas we utilize the SGD optimizer in 'Vision Control 100-200k' and 'Atari 100k'.

In the case of most baselines, we either adhere to the suggested hyperparameters provided by the authors of those baselines for each domain or fine-tune them to suit our setup when such suggestions are not available. Notably, SAC, DrQ-v2, and DreamerV3 employ a larger batch size of 512, while our method and EfficientZero achieve stable learning with a batch size of 256.

\begin{table}[h!]
\centering
\begin{tabular}{lcc}
% \begin{tabular}{
%   colspec = {| L{15em} | C{6em} | C{10em} |},
%   row{1} = {font=\bfseries},
% }
\toprule
\textbf{Name} & \textbf{Symbol} & \textbf{Value} \\
\midrule
\multicolumn{3}{l}{\textbf{General}} \\
\midrule
Replay capacity (FIFO) & --- & $10^6\!\!$ \\
Batch size & $\mathcal{B}$ & 256 \\
Discount & $\gamma$ & 0.997\\
Update-to-data (UTD)  & --- & 1 \\
Unroll steps & $l_{\text{unroll}}$ & 5 \\
TD steps & $k$ & 5 \\
Number of simulations in search & $N_{sim}$ & 32 (16 in `Atari 100k') \\
Number of sampled actions & $K$ & 16 (8 in `Atari 100k')\\
Self-play network updating interval & --- & 100 \\
Target network updating interval & --- & 400 \\
Starting steps when using SVE & $T_1$ & $4\cdot10^{4}$ \\
Threshold of buffer index when using SVE & $T_2$ & $2\cdot10^{4}$ \\
Priority exponent & $\alpha$ & 1 \\
Priority correction & $\beta$ & 1 \\
Reward loss coefficient  & $\lambda_1$ & 1.0 \\
Policy loss coefficient  & $\lambda_2$ & 1.0 \\
Value loss coefficient & $\lambda_3$ & 0.25 \\
Self-supervised consistency loss coefficient & $\lambda_4$ & 2.0 \\
Policy entropy loss coefficient  & --- & $5\cdot10^{-3}$ \\
\midrule
\multicolumn{3}{l}{\textbf{Proprio Control 50-100k }} \\
\midrule
Optimizer & --- & Adam \\
Optimizer: learning rate & --- & $3\cdot10^{-4}$ \\
Optimizer: weight decay & --- & $2\cdot10^{-5}$ \\
\midrule
\multicolumn{3}{l}{\textbf{Vision Control 100-200k \& Atari 100k}} \\
\midrule
Optimizer & --- & SGD \\
Optimizer: learning rate & --- & 0.2 \\
Optimizer: weight decay & --- & $1\cdot10^{-4}$ \\
Optimizer: momentum & --- & 0.9 \\
\bottomrule

\end{tabular}
\caption{hyper-parameters of EfficientZero V2.
}
\label{tab:hparams}
\end{table}

\subsection{Hyperparameters of Baselines}

\subsubsection{DreamerV3}
We use the official reimplementation of DreamerV3, which can be found at \href{https://github.com/danijar/dreamerv3}{https://github.com/danijar/dreamerv3}. In line with the original authors' recommendations, the results we used are based on their suggested hyperparameters and the S model size for Atari 100K, Proprio Control 50-100k, and Vision Control 100-200k. For a comprehensive list of hyperparameters, please refer to their published paper \citep{hafner2023mastering}.

\subsubsection{TD-MPC2}
We benchmark against the official implementation of TD-MPC2 available at
\href{https://github.com/nicklashansen/tdmpc2}{https://github.com/nicklashansen/tdmpc2}. 
We reproduce results according to the official implementation, which is shown in Fig. \ref{dmc_proprio}. We use the suggested hyperparameters and select the default 5M trainable parameters. Refer to their paper \citep{Anonymous2023TDMPC2} for a complete list of hyperparameters.
% Since its implementation is still under development up to the submission of our paper, the results for TD-MPC in Table \ref{tab:dmc_results_full} are obtained from the results directory at \href{https://github.com/nicklashansen/tdmpc2}{https://github.com/nicklashansen/tdmpc2}. Furthermore, 

\subsubsection{SAC}
We follow the SAC implementation from \href{https://github.com/ denisyarats/pytorch\_sac}{https://github.com/ denisyarats/pytorch\_sac}, and we use the hyperparameters suggested by the authors (when available). Refer to their repo for a complete list of hyperparameters. 
% We also sweep parameters to make it achieve better performance on our setting.

\subsubsection{DrQ-v2}
We follow the DrQ-v2 implementation from \href{https://github.com/facebookresearch/drqv2}{https://github.com/facebookresearch/drqv2}, and we use the hyperparameters suggested by the authors (when available). Refer to their repo for a complete list of hyperparameters. 
% We also sweep parameters to make it achieve better performance on our setting.

\subsubsection{EfficientZero}
We conduct benchmarks using the official reimplementation of EfficientZero, which can be found at \href{https://github.com/YeWR/EfficientZero}{https://github.com/YeWR/EfficientZero}. In line with the original authors' recommendations, we used their suggested hyperparameters for Atari 100K. For a comprehensive list of hyperparameters, please refer to their published paper \citep{ye2021mastering}.

\subsubsection{BBF}
We use the official results of BBF, which can be found at \href{https://github.com/google-research/google-research/tree/master/bigger\_better\_faster}{https://github.com/google-research/google-research/tree/master/bigger\_better\_faster}.

\section{Details of Experiments}
\label{train_curves}

\subsection{Comparison}
\label{compare}
We present the training curves across various benchmarks, including Atari 100k, Proprio Control 50-100k, and Vision Control 100-200k. Atari 100k, featuring 26 games, is a widely used benchmark for assessing the sample efficiency of different algorithms. In the case of Proprio Control 50-100k and Vision Control 100-200k, we have considered 20 continuous control tasks for each. You can find the training curves for EZ-V2 and the baselines in Figures \ref{dmc_proprio}, \ref{dmc_vision}, and \ref{atari}.

\subsection{Ablation}
\label{ablation}

Additionally, we include an ablation study focusing on the sampling-based Gumbel search and the mixed value target. Table \ref{statedm_abla}, \ref{visdm_abla} and \ref{atari_abla} demonstrate that our search method and mixed value target achieve superior performance on tasks with proprioceptive and image inputs. 
The action space is 1-dimensional in tasks such as Acrobot Swingup, Cartpole Swingup Sparse, and Pendulum Swingup. The dimension is greater than 2 in other DM-Control tasks. For Atari 100k, we selected 3 of the relatively most challenging tasks for our ablation study.

We have observed that our mixed value target outperforms other value estimation methods in most tasks, such as multi-step value target and generalized advantage estimation (GAE). This indicates that the method mitigating the off-policy issue consistently performs better. Compared to Sample MCTS, the S-Gumbel Search method significantly enhances performance in tasks with a high-dimensional action space. The results demonstrate that the S-Gumbel Search method can achieve superior performance with the limited simulations.

\begin{table}[h]
    \centering
    \caption{Proprio Control 50-100k}
    \label{statedm_abla}
    \scalebox{0.7}{
    \begin{tabular}{lcccccccc}
        % & \multicolumn{8}{c}{Environment} \\
        % \cmidrule{2-9}
        \toprule
        Method & Acrobot swingup & \begin{tabular}{c} Cartpole \\ swingup sparse\end{tabular} & Pendulum swingup & Reacher hard & Walker walk & Walker run & Quadruped walk \\
        \midrule
        Original EZ-v2 & \textbf{297.7} & \textbf{795.4} & 825.4 & \textbf{795.4} & \textbf{944.0} & \textbf{657.2} & \textbf{925.8} \\
        W/ Multi-step TD Target & 256.7 & 473.6 & \textbf{836} & 590.7 & 839.3 & 521.6 & 649.7 \\
        W/ GAE & 275.1 & 766.7 & 336 & 601.6 & 757.3 & 512.7 & 719.7 \\
        W/ Sample MCTS & 248.1 & 789.3 & 357.4 & 754.7 & 691.7 & 381.1 & 254.7 \\
        \bottomrule
    \end{tabular}
    }
\end{table}

\begin{table}[h]
    \centering
    \caption{Vision Control 100-200k}
    \label{visdm_abla}
    \scalebox{0.7}{
    \begin{tabular}{lcccccccc}
        \toprule
        % & \multicolumn{8}{c}{Environment} \\
        % \cmidrule{2-9}
        Method & Acrobot swingup & \begin{tabular}{c} Cartpole \\ swingup sparse\end{tabular} & Pendulum swingup & Reacher hard & Walker walk & Walker run & Quadruped walk \\
        \midrule
        Original EZ-v2 & \textbf{231.8} & 763.6 & \textbf{726.7} & \textbf{961.5} & \textbf{888.8} & 475.3 & \textbf{433.3} \\
        W/ Multi-step TD Target & 186.3 & 631.7 & 411.4 & 878 & 880.5 & \textbf{496.7} & 410.0 \\
        W/ GAE & 122.0 & \textbf{796.1} & 718.0 & 934.9 & 765.1 &491.5 & 299.3\\
        W/ Sample MCTS & 73.9 & 786.1 & 372.9 & 938.8 & 619.9 & 264.2 & 141.4 \\
        \bottomrule
    \end{tabular}
    }
\end{table}

\begin{table}[htbp]
    \centering
    \caption{Atari 100k}
    \label{atari_abla}
    \begin{tabular}{lccc}
        \toprule
        & \multicolumn{3}{c}{Game} \\
        \cmidrule{2-4}
        Method & Asterix & UpNDown & Qbert \\
        \midrule
        Original EZ-v2 & \textbf{61810.0} & 15224.3 & \textbf{16058.3} \\
        W/ Multi-step TD Target & 26023.1 & 13725.7 & 9463.3 \\
        W/ GAE & 22816.7 & \textbf{16255.3} & 7807.3 \\
        \bottomrule
    \end{tabular}
\end{table}

\subsection{Computation Load}
\label{load}

We have included practical comparisons between the TD-MPC2 and EZ-V2 algorithms in terms of computational load, using the 'Walker Run' task as an example. The following results include the parameter count, FLOPs (per decision step) and training time for each algorithm. The methods were benchmarked on a server equipped with 8 RTX 3090 graphics cards.

\begin{table}[htbp]
    \centering
    \caption{Model Parameters and Training Performance}
    \begin{tabular}{|c|c|c|c|}
        \hline
        & Parameters & Flops per decision step & Time per 100k training (h) \\
        \hline
        EZ-V2 & $\mathbf{1.3 \times 10^6}$ & $\mathbf{4.7 \times 10^7}$ & \textbf{2.7} \\
        TD-MPC2 & $4.9 \times 10^6$ & $3.6 \times 10^{10}$ & 3.3 \\
        \hline
    \end{tabular}
\end{table}

EZ-V2 requires 1000 times fewer FLOPs per decision step compared to TD-MPC2, while also using almost 4 times fewer parameters. The decision step, crucial for collecting interaction data, occurs during the evaluation. This remarkable efficiency arises from two primary factors:
\begin{itemize}
    \item The planning process in TD-MPC2, utilizing the MPPI method, involves predicting 9216 latent states. In contrast, our method extends only 32 latent states, significantly reducing the computational load.
    \item TD-MPC2 relies on an ensemble of Q-functions (5 heads), and its latent state dimension is 4 times larger than our method's, leading to higher computational demands.
\end{itemize}

The FLOPs per decision step are vital for deployments, especially in robotics control. Since robots typically have limited edge computing resources, using less computation is essential for real-time tasks. TD-MPC2 needs more computational resources and high-performance computing. In contrast, our method runs efficiently with less demand of computing resources, making it ideal for real-time robotic planning.

Regarding training time, both methods require a similar amount of time per 100k training steps. This similarity in time consumption is due to that EZ-V2 and TD-MPC2 share similar unrolled training frameworks with the same batch size. EZ-V2 wins in speed due to its distributed implementation, without rollout and evaluation overheads in the training process.

\begin{figure}[t]
\centering
\includegraphics[width=0.9\textwidth]{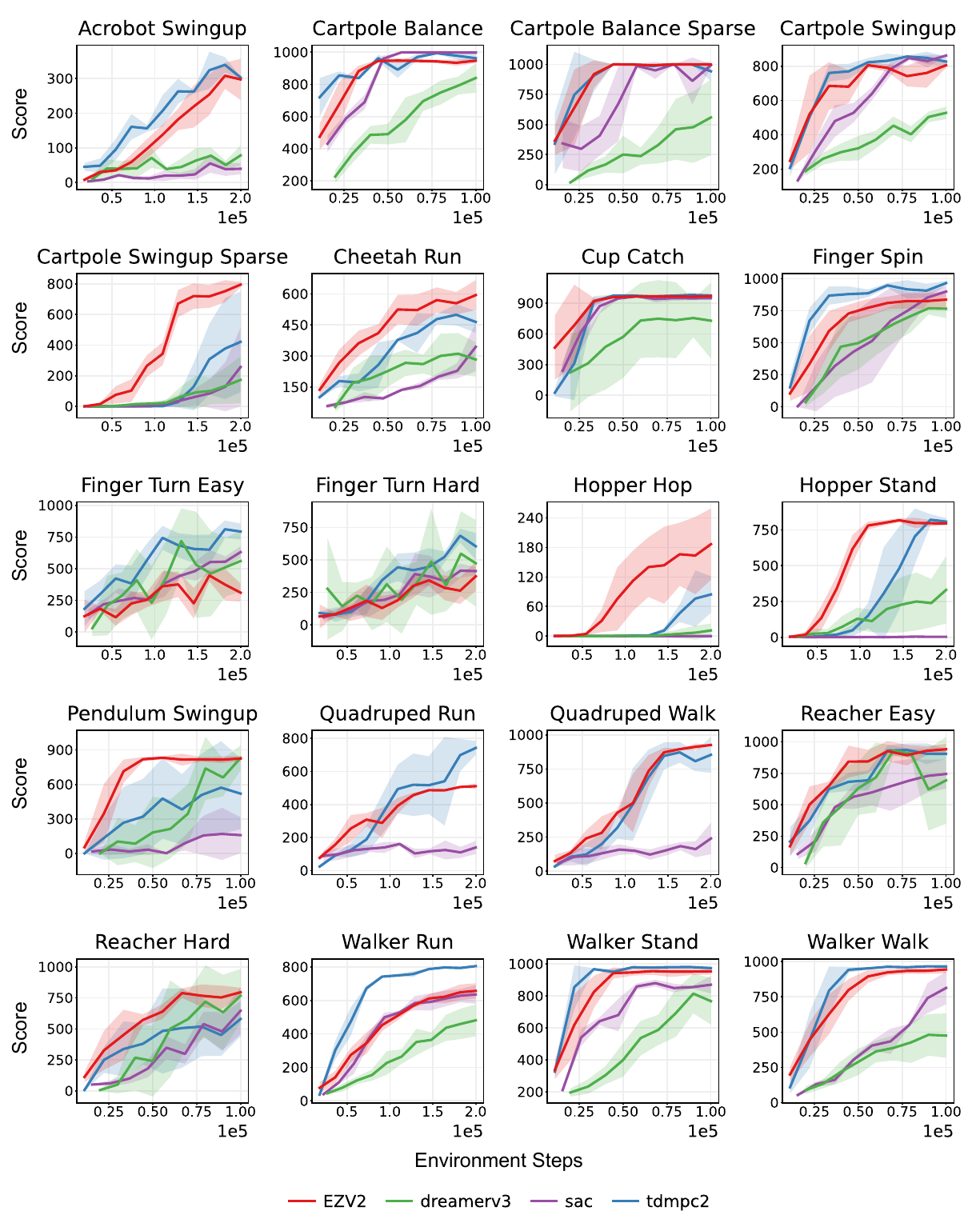}
\caption{DMC scores for proprioceptive inputs with a budget of 200K frames. it corresponds to 100K steps due to the action repeat. (Because different algorithms have varying logging frequencies, the starting points are not the same.)}
\label{dmc_proprio}
\end{figure}

\begin{figure}[t]
\centering
\includegraphics[width=0.9\textwidth]{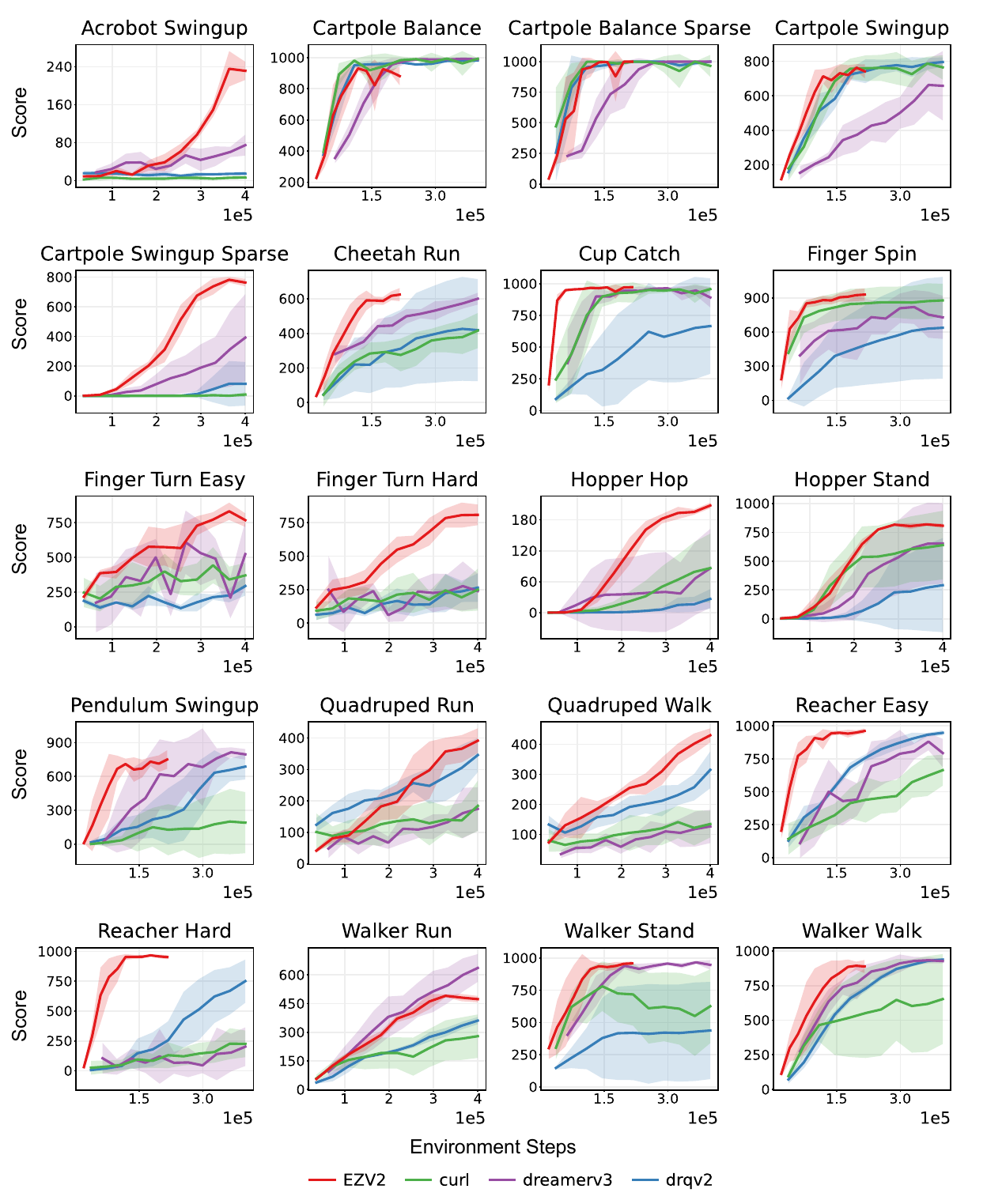}
\caption{DMC scores for image inputs with a budget of 400K frames. It corresponds to 200K steps due to the action repeat. (Because different algorithms have varying logging frequencies, the starting points are not the same.)}
\label{dmc_vision}
\end{figure}

\begin{figure}[t]
\centering
\includegraphics[width=0.9\textwidth]{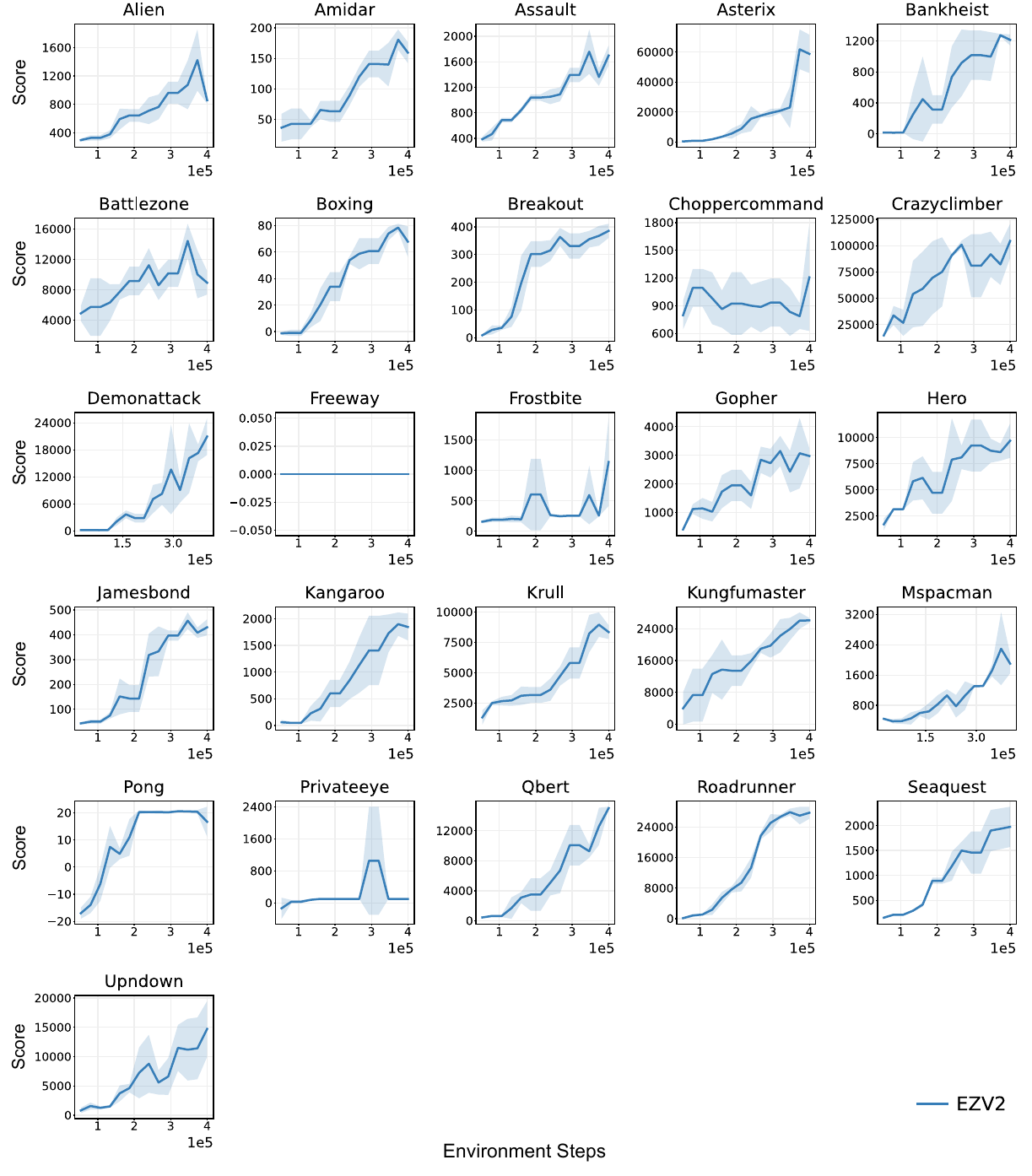}
\caption{Atari training curves with a budget of 400K frames, amounting to 100K interaction data.}
\label{atari}
\end{figure}

% \begin{figure}[t]
% \centering
% \includegraphics[width=0.75\textwidth]{sections/figs/ablation_search.pdf}
% \caption{The ablation study of our search method, namely sampling-based Gumbel search (S-Gumbel Search). Besides, we compare with sample MCTS \citep{hubert2021learning}, and our search method with different simulations.}
% \label{ablation_search}
% \end{figure}

% \begin{figure}[t]
% \centering
% \includegraphics[width=0.75\textwidth]{sections/figs/ablation_value.pdf}
% \caption{The ablation study of the mixed value target we propose. Besides, we compare with multi-step TD target and double value target. Double Q-value target is similar to the value target in DDPG \cite{lillicrap2015continuous} (more details can be found in Section \ref{ablation_sec}).}
% \label{ablation_value}
% \end{figure}

% You can have as much text here as you want. The main body must be at most $8$ pages long.
% For the final version, one more page can be added.
% If you want, you can use an appendix like this one, even using the one-column format.
%%%%%%%%%%%%%%%%%%%%%%%%%%%%%%%%%%%%%%%%%%%%%%%%%%%%%%%%%%%%%%%%%%%%%%%%%%%%%%%
%%%%%%%%%%%%%%%%%%%%%%%%%%%%%%%%%%%%%%%%%%%%%%%%%%%%%%%%%%%%%%%%%%%%%%%%%%%%%%%

\end{document}